\documentclass[12pt,reqno]{amsart}
\usepackage{amsaddr}

\usepackage[top=30truemm,bottom=30truemm,left=25truemm,right=25truemm]{geometry}

\usepackage{amsmath,amssymb,amsthm,mathabx}
\usepackage{bm}
\usepackage{graphicx,here,comment}
\usepackage{microtype}
\usepackage{subfigure}
\usepackage{booktabs} 

\usepackage{amsmath,amssymb,amsthm,mathabx,amsfonts}
\usepackage{nicefrac}       
\usepackage{bm}
\usepackage{graphicx,here,comment,float,wrapfig}
\usepackage{microtype}
\usepackage{subfigure}
\usepackage{booktabs} 

\newtheorem{theorem}{Theorem}
\newtheorem{corollary}{Corollary}
\newtheorem{lemma}{Lemma}
\newtheorem{proposition}{Proposition}

\theoremstyle{definition}

\newcommand{\R}{\mathbb{R}}
\newcommand{\N}{\mathbb{N}}

\newcommand{\mF}{\mathcal{F}}

\newcommand{\mD}{\mathcal{D}}
\newcommand{\mI}{\mathcal{I}}

\newcommand{\mH}{\mathcal{H}}
\newcommand{\mS}{\mathbb{S}}

\newcommand{\mN}{\mathcal{N}}

\newcommand{\mR}{\mathcal{R}}
\newcommand{\Ep}{\mathbb{E}}
\renewcommand{\Pr}{\mathrm{Pr}}

\newcommand{\mE}{\mathcal{E}}

\renewcommand{\hat}{\widehat}
\renewcommand{\tilde}{\widetilde}
\renewcommand{\check}{\widecheck}
\newcommand{\mone}{\textit{\textbf{1}}}
\newcommand{\argmin}{\operatornamewithlimits{argmin}}
\newcommand{\argmax}{\operatornamewithlimits{argmax}}
\newcommand{\inte}{\operatornamewithlimits{Int}}
\newcommand{\vect}{\operatornamewithlimits{vec}}

\usepackage{color}

\title[DNNs for Non-Smooth Functions]{Deep Neural Networks Learn\\Non-Smooth Functions Effectively}

\date{\today}
\author{Masaaki Imaizumi \and Kenji Fukumizu}
\address{The Institute of Statistical Mathematics}
\email{imaizumi@ism.ac.jp}
\begin{document}

\maketitle

\begin{abstract}
We theoretically discuss why deep neural networks (DNNs) performs better than other models in some cases by investigating statistical properties of DNNs for non-smooth functions.
While DNNs have empirically shown higher performance than other standard methods, understanding its mechanism is still a challenging problem.  From an aspect of the statistical theory, it is known many standard methods attain the optimal rate of generalization errors for smooth functions in large sample asymptotics, and thus it has not been straightforward to find theoretical advantages of DNNs.  This paper fills this gap by considering learning of a certain class of non-smooth functions, which was not covered by the previous theory.  We derive the generalization error of estimators by DNNs with a ReLU activation, and show that convergence rates of the generalization by DNNs are almost optimal to estimate the non-smooth functions, while some of the popular models do not attain the optimal rate. 
In addition, our theoretical result provides guidelines for selecting an appropriate number of layers and edges of DNNs.
We provide numerical experiments to support the theoretical results.
\end{abstract}


\section{Introduction} \label{submission} 

\textit{Deep neural networks} (DNNs) have shown outstanding performance on various tasks of data analysis \cite{schmidhuber2015deep, lecun2015deep}.
Enjoying their flexible modeling by a multi-layer structure and many elaborate computational and optimization techniques, DNNs empirically achieve higher accuracy than many other machine learning methods such as kernel methods \cite{hinton2006fast, le2011optimization, kingma14adam}.
Hence, DNNs are employed in many successful applications, such as image analysis \cite{he2016deep}, medical data analysis \cite{fakoor2013using}, natural language processing \cite{collobert2008unified}, and others.

Despite such outstanding performance of DNNs, little is yet known why DNNs outperform the other methods.
Without sufficient understanding, practical use of DNNs could be inefficient or unreliable.
To reveal the mechanism, numerous studies have investigated theoretical properties of neural networks from various aspects.
The approximation theory has analyzed the expressive power of neural networks \cite{cybenko1989approximation, barron1993universal, bengio2011expressive, montufar2014number, yarotsky2017error, petersen2017optimal,bolcskei2017optimal}, the statistical learning theory elucidated generalization errors \cite{barron1994approximation, neyshabur2015norm, schmidt2017nonparametric, zhang2017understanding, suzuki2017fast}, and the optimization theory has discussed the landscape of the objective function and dynamics of learning\cite{baldi1989neural,fukumizu2000local, dauphin2014identifying, kawaguchi2016deep, soudry2016no}.

One limitation in the existing statistical analysis of DNNs is a  \textit{smoothness assumption} for data generating processes,   
which requires that data $\{(Y_i,X_i)\}$ are given by 
\begin{align*}
    Y_i = f(X_i) + \xi_i, ~\xi_i\sim \mN(0,\sigma^2),
\end{align*}
where $f$ is a $\beta$-times differentiable function with $D$-dimensional input.  
With this setting, however, many popular methods such as kernel methods, Gaussian processes, series methods, and so on, as well as DNNs, achieve a bound for generalization errors as
\begin{align*}
    O \left( n^{-2\beta/(2\beta + D)} \right), \qquad (n\to\infty).
\end{align*}
This is known to be a minimax optimal rate of generalization with respect to sample size $n$ \cite{stone1982optimal, tsybakov2003introduction, gine2015mathematical}, and 
hence, as long as we employ the smoothness assumption, it is not easy to show a theoretical evidence for the empirical advantage of DNNs.

This paper considers estimation of \textit{non-smooth} functions for the data generating processes to break the difficulty.  
Specifically, we discuss a nonparametric regression problem with a class of \textit{piecewise smooth functions}, which may be non-smooth on the boundaries of pieces in their domains.
Then, we derive a rate of generalization errors with the least square and Bayes estimators by DNNs of the ReLU activation as
\begin{align*}
    O \left(\max\left\{ n^{-2\beta/(2\beta + D)} , n^{-\alpha/(\alpha + D-1)} \right\} \right),\qquad (n\to\infty)
\end{align*}
up to log factors (Theorems \ref{thm:non-bayes}, \ref{thm:bayes}, and Corollary \ref{cor:bayes}).
Here, $\alpha$ and $\beta$ denote a degree of smoothness of piecewise smooth functions, and $D$ is the dimensionality of inputs.
We prove also that this rate of generalizations by DNNs is optimal in the minimax sense (Theorem \ref{thm:minimax}).
In addition, we show that some of other standard methods, such as kernel methods and orthogonal series methods, are not able to achieve this optimal rate.  Our results thus show that DNNs certainly have a theoretical advantage under the non-smooth setting. We will provide some numerical results supporting our results.

The contributions of this paper are as follows:
    \begin{itemize}
        \setlength{\parskip}{0cm}
        \setlength{\itemsep}{0cm}
        \item We derive a rate of convergence of the generalzation errros in the estimators by DNNs for the class of piecewise smooth functions.
        Our convergence results are more general than existing studies, since the class is regarded as a generalization of smooth functions.
        \item We prove that DNNs theoretically outperform other standard methods for data from non-smooth generating processes, as a consequence the proved convergence rate of generalization error.
        \item We provide a practical guideline on the structure of DNNs; namely, we show a necessary number of layers and parameters of DNNs to achieve the rate of convergence.
        It is shown in Table \ref{table:architecture}.
    \end{itemize}

All of the proofs are deferred to the supplementary material.

\begin{table}[htbp]
\vskip 0.15in
\begin{center}
\begin{small}
\begin{sc}
\begin{tabular}{lc}
\toprule
Element & Number \\
\midrule
\# of layers    & $ \leq c  (1+\max\{\beta/D,\alpha/2(D-1)\})$\\
\# of parameters      & $ c' n^{\max\{D/(2\beta + D),(D-1)/(\alpha + D - 1)\}}$\\
\bottomrule
\end{tabular}
\end{sc}
\end{small}
\end{center}

\caption{Architecture for DNNs which are necessary to achieve the optimal rate of generalization errors. $c,c'>0$ are some constants.\label{table:architecture}}
\end{table}

\subsection{Notation}

We use notations $I := [0,1]$ and 
$\N$ for natural numbers.
The $j$-th element of vector $b$ is denoted by $b_{j}$, and   $\|\cdot\|_q := (\sum_j b_j^q)^{1/q}$ is the $q$-norm ($q \in [0,\infty]$).
$\vect(\cdot)$ is a vectorization operator for matrices. 
For $z \in \N$, $[z] := \{1,2,\ldots,z\}$ is the set of positive integers no more than $z$.
For a measure $P$ on $I$ and a function $f:I \to \R$, $\|f\|_{L^2(P)}:= ( \int_I |f(x)|^2dP(x))^{1/2}$ denotes the $L^2(P)$ norm.  
$\otimes$ denotes a tensor product, and $\bigotimes_{j \in  [J]}x_j := x_1 \otimes \cdots \otimes x_J$ for a sequence $\{x_j\}_{j \in [J]}$.


\section{Regression with Non-Smooth Functions}


\subsection{Regression Problem}

In this paper, we use the $D$-dimensional cube $I^D$ ($D \geq 2$) for the domain of data.
Suppose we have a set of observations $(X_i,Y_i) \in  I^D \times \R$ for $i \in [n]$ which is independently and identically distributed with the data generating process
\begin{align}
    Y_i = f^*(X_i) + \xi_i, \label{eq:reg}
\end{align}
where $f^* :I^D \to \R$ is an unknown true function and $\xi_i$ is  Gaussian noise with mean $0$ and variance $\sigma^2 > 0$ for $i \in [n]$.
We assume that the marginal distribution of $X$ on $I^D$ has a positive and bounded density function $P_X(x)$.

The goal of the regression problem is to estimate $f^*$ from the set of observations $\mD_n := \{(X_i,Y_i)\}_{i \in [n]}$.
With an estimator $\hat{f}$, its performance is measured by the $L^2(P_X)$ norm: $\|\hat{f} - f^*\|_{L^2(P_X)}^2 = \Ep_{X \sim P_X} [ ( \hat{f}(X) - f^*(X))^2 ]$.
There are various methods to estimate $f^*$ and their statistical properties are extensively investigated (For summary, see \cite{wasserman2006all} and \cite{tsybakov2003introduction}).

A classification problem can be also analyzed through the regression framework.
For instance, consider a $Q$-classes classification problem with covariates $X_i$ and labels $Z_i \in [Q]$ for $i \in [n]$.
To describe the classification problem, we consider a $Q$-dimensional vector-valued function $f^*(x) = (f_1^*(x) ,...,f_Q^*(x))$ and a generative model for $Z_i$  as $Z_i = \argmax_{q \in [Q]} f_q^*(X_i)$.
Estimating $f^*$ can solve the classification problem.
(For summary, see \cite{steinwart2008support}).

\subsection{Definition of Non-Smooth Function}

\begin{figure}
\begin{center}
\includegraphics[width=0.5\hsize]{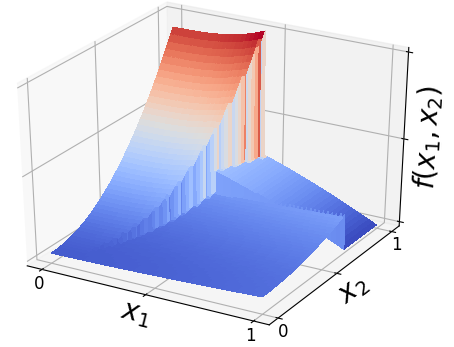}
\caption{An example of piecewise smooth functions with a $2$-dimensional support. The support is divided into three pieces and the function is discontinuous on their boundaries. \label{fig:p-smooth}}
\end{center}
\end{figure}

To describe non-smoothness of functions, we introduce a notion of \textit{piecewise smooth functions} which have a support divided into several pieces and smooth only within each of the pieces.
On boundaries of the pieces, piecewise smooth functions are non-smooth, i.e. non-differentiable and even discontinuous.
Figure \ref{fig:p-smooth} shows an example of piecewise smooth functions.

As preparation, we introduce notions of (i) smooth functions and (ii) pieces in supports, and  (iii) piecewise smooth functions.  In the sequel, for a set $R\subset I^D$, let $\mone_{R}:I^D \to \{0,1\}$ denote the indicator function of $R$; i.e., $\mone_{R}(x) = 1$ if $x \in R$, and $\mone_{R}(x) = 1$ otherwise.

\textbf{(i). Smooth Functions}

We use the standard \textit{the H\"older space} for smooth functions.
With a domain  $\Omega$ in $\R^D$, let $H^{\beta}(\Omega)$ be the H\"older space on $\Omega$, which is the space of functions $f:\Omega \to \R$ such that they are $\lfloor \beta \rfloor$-times differentiable and the $\lfloor \beta \rfloor$-th derivatives are $\beta- \lfloor \beta \rfloor$-H\"older continuous.
For the precise definition, see Appendix \ref{sec:holder} or \cite{gine2015mathematical}.

\textbf{(ii). Pieces in Supports}

To describe pieces in supports,  we use the class of {\em horizon functions} \cite{petersen2017optimal}. 
Preliminarily, we consider a smooth function $h\in H^\alpha(I^{D-1})$.
Then, we define a horizon function $\Psi_h:I^D \to \{0,1\}$ as
\begin{align*}
    \Psi_h := \Psi(x_1,\ldots,x_{d-1}, x_d \pm  h(x_1,\ldots,x_{d-1},x_{d+1},\ldots,x_D),x_{d+1},...,x_D),
\end{align*}
for some $d \in [D]$, where $\Psi:I^D \to \{0,1\}$ is the Heaviside function such that $\Psi(x) = \mone_{\{x \in I^D \mid x_d \geq 0\}}$.

For each horizon function, we define a \textit{basis piece} $A \subset I^D$ such that there exist $\Psi_h$ such that
\begin{align*}
    A = \left\{ x \in I^D \mid \Psi_h(x)=1 \right\}.
\end{align*}
A basis piece is regarded as one side of surfaces by $h$.
Additionally, we introduce a restrict for $A$ as a transformed sphere, namely, we consider $\Psi_h$ such that there exists an $\alpha$-smooth embedding $e:\{x \in \R^D \mid \|x\|_2 \leq 1\} \to \R^D$ satisfying $A = I^D \cap \mbox{Image}(e)$ (detail is provided in Appendix \ref{sec:holder}).
The notion of basis pieces is an extended version of the boundary fragment class \cite{dudley1974metric,mammen1999smooth} which is dense in a class of all convex sets in $I^D$ when $\alpha = 2$.

We define a {\em piece} by the intersection of $J$ basis pieces; namely, the set of pieces is defined by 
\begin{align*}
    \mR_{\alpha,J} := \left\{  R\subset [0,1]^D\mid \text{there are basic pieces } A_1,\ldots,A_J \text{ such that } R = \cap_{j=1}^J A_j \right\}.
\end{align*}

Intuitively, $R \in \mR_{\alpha,J}$ is a set with piecewise $\alpha$-smooth boundaries.
Also, by considering intersections of $J$ basis pieces, $\mR_{\alpha,J}$ contains a set with non-smooth boundaries.
In Figure \ref{fig:p-smooth}, there are three pieces from $\mR_{\alpha,J}$ in the support of the function.

\textbf{(iii). Piecewise Smooth Functions}

Using $H^\beta(I^D)$ and $\mR_{\alpha,J}$,  we define piecewise smooth functions.
Let $M \in \N$ be a finite number of pieces of the support $I^D$.
We introduce the set of piecewise smooth functions by 
\begin{align*}
    \mF_{M,J,\alpha,\beta} := \left\{ \sum_{m=1}^M f_m \otimes \mone_{R_m} : f_m \in H^{\beta}(I^D), R_m \in \mR_{\alpha,J} \right\}.
\end{align*}
Since $f_m(x)$ realizes only when $x \in R_m$, the notion of $\mF_{M,J,\alpha,\beta}$ can express a combination of smooth functions on each piece $R_m$.
Hence, functions in $\mF_{M,J,\alpha,\beta}$ are non-smooth (and even discontinuous) on boundaries of $R_m$.
Obviously, $H^\beta(I^D) \subset \mF_{M,J,\alpha,\beta}$ with $M=1$ and $R_1 = I^D$, hence the notion of piecewise smooth functions can describe a wider class of functions.

\section{Estimators with DNNs and Theoretical Analysis}


\subsection{Models by Deep Neural Networks}

Let $L \in \N$ be the number of layers in DNNs.
For $\ell \in [L+1]$, let $D_\ell \in \N$ be the dimensionality of variables in the $\ell$-th layer.
For brevity, we set $D_{L+1}=1$, i.e., the output is one-dimensional.
We define $A_\ell \in \R^{D_{\ell + 1} \times D_\ell}$ and $b_\ell \in \R^{D_\ell}$ be matrix and vector parameters to give the transform of $\ell$-th layer.
The \textit{architecture} $\Theta$ of DNN is a set of $L$ pairs of $(A_\ell,b_\ell)$:
\begin{align*}
    \Theta := ((A_1,b_1),...,(A_{L},b_{L})).
\end{align*}
We define $|\Theta| := L$ be a number of layers in $\Theta$, $\|\Theta\|_{0} := \sum_{\ell \in [L]} \|\vect (A_\ell)\|_0 + \|b_\ell\|_0$ as a number of non-zero elements in $\Theta$, and $\|\Theta\|_\infty := \max\{\max_{\ell \in [L]} \|\vect (A_\ell)\|_\infty,\max_{\ell \in [L]} \|b_\ell\|_\infty\}$ be the largest absolute value of the parameters in $\Theta$.

For an activation function $\eta: \R^{D'} \to \R^{D'}$ for each $D' \in \N$, this paper considers the ReLU activation $\eta(x) = (\max\{x_d,0\})_{d\in [D']}$.

The model of neural networks with architecture $\Theta$ and activation $\eta$ is the function $G_\eta[\Theta] : \R^{D_1} \to \R$, which is defined inductively as 
\begin{align*}
    G_\eta[\Theta](x) = x^{(L+1)},
\end{align*}
and it is inductively defined as
\begin{align*}
    &x^{(1)} := x,~~~~x^{(\ell + 1)} := \eta(A_{\ell} x^{(\ell)} +  b_{\ell}), \mbox{~for~}\ell \in [L],
\end{align*}
where $L=|\Theta|$ is the number of layers.
The set of model functions by DNNs is thus given by 
\begin{align*}
    &\mF_{NN,\eta}(S,B,L') := \Bigl\{ G_\eta[\Theta]:I^D \to \R \mid \|\Theta\|_{0} \leq S, \|\Theta\|_{\infty} \leq B, |\Theta| \leq L' \Bigr\},
\end{align*}
with $S \in \N$, $B > 0$, and $L' \in \N$.
Here, $S$ bounds the number of non-zero parameters of DNNs by $\Theta$, namely, the number of edges of an architecture in the networks.
This also describes sparseness of DNNs. 
$B$ is a bound for scales of parameters.


\subsection{Least Square Estimator}

Using the model of DNNs, we define a least square estimator by empirical risk minimization.
Using the observations $\mD_n$, we consider the minimization problem with respect to parameters of DNNs as
\begin{align}
    \hat{f}^{L} \in \argmin_{f \in \mF_{NN,\eta}(S,B,L)} \frac{1}{n}\sum_{i \in [n]} (Y_i - f(X_i))^2, \label{opt:erm}
\end{align}
and use $\hat{f}^{L}$ for an estimator of $f^*$.

Note that the problem \eqref{opt:erm} has at least one minimizer since the parameter set $\Theta$ is compact and $\eta$ is continuous.
If necessary, we can add a regularization term for the problem \eqref{opt:erm}, because it is not difficult to extend our results to an estimator with regularization.
Furthermore, we can apply the early stopping techniques, since they play a role as the regularization \cite{lecun2015deep}.
However, for simplicity, we confine our arguments of this paper in the least square.

We investigate theoretical aspects of convergence properties of $\hat{f}^{L}$ with a ReLU activation.
\begin{theorem} \label{thm:non-bayes}
    Suppose $f^* \in \mF_{M,J,\alpha,\beta}$.
    Then, there exist constants $c_1,c'_1,C_{L}>0, s \in \N\backslash \{1\}$, and $(S,B,L)$ satisfying
    \begin{enumerate}
        \setlength{\parskip}{0cm}
          \setlength{\itemsep}{0cm}
        \item[(i)] $S = c'_1 \max\{n^{D/(2\beta + D)}, n^{(D-1)/(\alpha + D - 1)}\}$,
        \item [(ii)] $B \geq c_1n^{s}$,
        \item [(iii)] $L \leq c_1  (1+\max\{\beta/D,\alpha/2(D-1)\})$,
    \end{enumerate}
    such that $\hat{f}^{L} \in \mF_{NN,\eta}(S,B,L)$ provides
    \begin{align}
        &\|\hat{f}^{L}-f^*\|_{L^2(P_X)}^2 \leq C_{L} M^2 J^2 \max\{n^{-2\beta/(2\beta + D)}, n^{-\alpha/(\alpha + D - 1)}\} (\log n)^2, \label{ineq:thm_non_bayes}
    \end{align}
    with probability at least $1-c_1 n^{-2}$. 
\end{theorem}
Proof of Theorem \ref{thm:non-bayes} is a combination of a set estimation \cite{dudley1974metric, mammen1995asymptotical}, an approximation theory of DNNs \cite{yarotsky2017error, petersen2017optimal,bolcskei2017memory}, and an applications of the empirical process techniques \cite{koltchinskii2006local, gine2015mathematical, suzuki2017fast}.


The rate of convergence in Theorem \ref{thm:non-bayes} is simply interpreted as follows.
The first term $n^{-2\beta/(2\beta + D)}$ describes an effect of estimating $f_m \in H^\beta(I^D)$ for $m \in [M]$.
The rate corresponds to the minimax optimal convergence rate of generalization errors for estimating smooth functions in $H^\beta(I^D)$ (For a summary, see \cite{tsybakov2003introduction}).
The second term $n^{-\alpha / (\alpha + D-1)}$ reveals an effect from estimation of $\mone_{R_m}$ for $m \in [M]$ through estimating the boundaries of $R_m \in \mR_{\alpha,J}$.
The same rate of convergence appears in a problem for estimating sets with smooth boundaries \cite{mammen1995asymptotical}.

We remark that a larger number of layers decreases $B$.
Considering the result by \cite{bartlett1998sample}, which shows that large values of parameters make the performance of DNNs worse, the above theoretical result suggests that a deep structure can avoid the performance loss caused by large parameters.   

We can consider an error from optimization independent to the statistical generalization.
The following corollary provides the statement.
\begin{corollary}\label{cor:opt}
    If a learning algorithm outputs $\check{f}^L \in \mF_{NN,\eta}(S,B,L)$ such that
    \begin{align*}
         n^{-1}\sum_{i \in [n]}(Y_i - \check{f}^L(X_i))^2 - (Y_i - \hat{f}^L(X_i))^2 \leq \Delta_n,
    \end{align*}
    with a positive parameter $\Delta_n$, then the following holds:
\begin{align*}
     &\Ep_{f^*} \left[ \|\check{f}^L-f^*\|_{L^2(P_X)}^2 \right] \leq C_{L} \max\{n^{-2\beta/(2\beta + D)}, n^{-\alpha/(\alpha + D - 1)}\} \log n + \Delta_n.
\end{align*}
\end{corollary}
Here, $\Ep_{f^*}[\cdot]$ denotes an expectation with respect to the true distribution of $(X,Y)$.
Applying results on the magnitude of $\Delta$ (e.g. \cite{kawaguchi2016deep}), we can evaluate generalization including optimization errors.

\subsection{Bayes Estimator}

We define a Bayes estimator for DNNs which can avoid the non-convexity problem in optimization.
Fix architecture $\Theta$ and $\mF_{NN,\eta}(S,B,L)$ with given $S,B$ and $L$.
Then, a prior distribution for $\mF_{NN,\eta}(S, B, L)$ is defined through providing distributions for the parameters contained in $\Theta$.
Let $\Pi_{\ell}^{(A)}$ and $\Pi_{\ell}^{(b)}$ be distributions of $A_\ell$ and $b_\ell$ as $A_\ell \sim \Pi_{\ell}^{(A)}$ and $b_\ell \sim \Pi_{\ell}^{(b)}$ for $\ell \in [L]$.
We set $\Pi_{\ell}^{(A)}$ and $\Pi_{\ell}^{(b)}$ such that each of the $S$ parameters of $\Theta$ is uniformly distributed on $[-B, B]$, and the other parameters degenerate at $0$.
Using these distributions, we define a prior distribution $\Pi_\Theta$ on $\Theta$ by $\Pi_\Theta := \bigotimes_{\ell \in [L]} \Pi_{\ell}^{(A)} \otimes \Pi_{\ell}^{(b)}$.
Then, a prior distribution for $f \in \mF_{NN,\eta}(S,B,L)$ is defined by 
\begin{align*}
    \Pi_f(f) := \Pi_\Theta(\Theta :  G_{\eta}[\Theta] = f).
\end{align*}

We consider the posterior distribution for $f$.
Since the noise $\xi_i$ in \eqref{eq:reg} is Gaussian with its variance $\sigma^2$, the posterior distribution is given by 
\begin{align*}
    &d\Pi_f(f | \mD_n) = \frac{\exp(-\sum_{i \in [n]}(Y_i - f(X_i))^2/\sigma^{2})d\Pi_f(f)}{\int \exp(-\sum_{i \in [n]}(Y_i - f'(X_i))^2/\sigma^{2})d\Pi_f(f')}.
\end{align*}

Note that we do not discuss computational issues of the Bayesian approach since the main focus is a theoretical aspect.
To solve the computational problems, see \cite{hernandez2015probabilistic} and others.

We provide theoretical analysis on the speed of contraction for the posterior distribution.
Same as the least square estimator cases, we consider a ReLU activation function.
\begin{theorem} \label{thm:bayes}
    Suppose $f^* \in \mF_{M,J,\alpha,\beta}$.
    Then, there exist constants $c_2,c'_2,C_{B}>0, s \in \N \backslash \{1\}$, architecture $\Theta: \|\Theta\|_{0} \leq S, \|\Theta\|_{\infty} \leq B, |\Theta| \leq L$ satisfying following conditions:
    \begin{enumerate}
        \setlength{\parskip}{0cm}
          \setlength{\itemsep}{0cm}
        \item[(i)] $S = c'_2 \max\{n^{D/(2\beta + D)}, n^{(2D-2)/(2\alpha + 2D - 2)}\}$,
        \item [(ii)] $B \geq c_{2}n^{s}$,
        \item [(iii)] $L \leq c_{2}  (1+\max\{\beta/D,\alpha/2(D-1)\})$,
    \end{enumerate}
     and a prior distribution $\Pi_f$, such that the  posterior distribution $\Pi_f(\cdot | \mD_n)$ provides
    \begin{align}
        &\Ep_{f^*} \Bigl[ \Pi_f \Bigl( f: \| f-f^*\|_{L^2(P_X)}^2 \geq rC_B \max\{n^{-2\beta/(2\beta + D)}, n^{-\alpha/(\alpha + D - 1)}\} (\log n)^2  | \mD_n \Bigr) \Bigr] \notag \\
        &\leq \exp \left( -r^2c_2\max\{n^{D/(2\beta + D)}, n^{(D-1)/(\alpha + D - 1)}\}  \right),\label{ineq:thm_bayes} 
    \end{align}
    for all $r > 0$.
\end{theorem}
To provide proof of Theorem \ref{thm:bayes}, we additionally apply studies for statistical analysis for Bayesian nonparametrics \cite{van2008rates, vaart2011information}.

Based on the result, we define a Bayes estimator as $\hat{f}^B := \int f d \Pi_f (f | \mD_n)$, 
by the Bochner integral in $L^\infty(I^D)$.
Then, we obtain the rate of convergence of $\hat{f}^B$ by the following corollary.
\begin{corollary} \label{cor:bayes}
    With the same setting in Theorem \ref{thm:bayes}, consider $\hat{f}^B$.
    Then, we have
    \begin{align*}
        &\Ep_{f^*} \left[ \|\hat{f}^{B}-f^*\|_{L^2(P_X)}^2 \right] \leq C_{B} \max\{n^{-2\beta/(2\beta + D)}, n^{-\alpha/(\alpha + D - 1)}\} (\log n)^2.
    \end{align*}
\end{corollary}

This result states that the Bayes estimator can achieve the same rate as the least square estimator shown in Theorem \ref{thm:non-bayes}.
Since the Bayes estimator does not use optimization, we can avoid the non-convex optimization problem, while the computation of the posterior and mean are not straightforward.

\section{Discussion: Why DNNs work better?}


\subsection{Optimality of the DNN Estimators}

We show optimality of the rate of convergence by the DNN estimators in Theorem \ref{thm:non-bayes} and Corollary \ref{cor:bayes}.
We employ a theory of minimax optimal rate which is known in the field of mathematical statistics \cite{gine2015mathematical}.
The theory derives a lower bound of a convergence rate with arbitrary estimators, thus we can obtain a theoretical limitation of convergence rates.

The following theorem shows the minimax optimal rate of convergence for the class of piecewise smooth functions $\mathcal{\mF}_{M,J,\alpha,\beta}$.
\begin{theorem}\label{thm:minimax}
    Consider $\bar{f}$ is an arbitrary estimator for $f^* \in \mF_{M,J,\alpha,\beta}$.
    Then, there exists a constant $C_{mm} > 0$ such that
    \begin{align*}
        &\inf_{\bar{f}} \sup_{f^* \in \mathcal{\mF}_{M,J,\alpha,\beta}} \Ep_{f^*} \left[ \|\bar{f} - f^*\|_{L^2(P_X)}^2\right]  \geq C_{mm} \max\left\{n^{-2\beta/(2\beta + D)}, n^{-\alpha / (\alpha + D-1)}\right\}.
    \end{align*}
\end{theorem}
Proof of Theorem \ref{thm:minimax} employs techniques in  the minimax theory developed by \cite{yang1999information} and \cite{raskutti2012minimax}.

We show that the rate of convergence by the estimators with DNNs are optimal in the minimax sense, since the rates in Theorems \ref{thm:non-bayes} and \ref{thm:bayes} correspond to the lower bound of Theorem \ref{thm:minimax} up to a log factor.
In other words, for estimating $f^* \in \mF_{M, J,\alpha,\beta}$, no other methods could achieve a better rate than the estimators by DNNs.

\subsection{Non-Optimality of Other Methods}

We discuss non-optimality of some of other standard methods to estimate piecewise smooth functions.
To this end, we consider a class of \textit{linear estimators}.
The class contains any estimators with the following formulation:
\begin{align}
    \hat{f}^{\mathrm{lin}} (x) = \sum_{i \in [n]} \Psi_i(x; X_1,...,X_n) Y_i, \label{def:lin}
\end{align}
where $\Psi_i$ is an arbitrary function which depends on $X_1,...,X_n$.
Various estimators are regarded as linear estimators, for examples, kernel methods, Fourier estimators, splines, Gaussian process, and others.

A study for nonparametric statistics (Section 6 in \cite{korostelev2012minimax}) proves inefficiency of  linear estimators with non-smooth functions.
Based on the result, the following corollary holds:
\begin{corollary}
    Linear estimators \eqref{def:lin} do not attain the optimal rate in Theorem \ref{thm:minimax} for $\mF_{M,J,\alpha,\beta}$.\\
   Hence, there exist $f^* \in \mF_{M,J,\alpha,\beta}$ such that $\hat{f} \in \{\hat{f}^L,\hat{f}^B\}$ and any $\hat{f}^{\mathrm{lin}}$, large $n$ provides
    \begin{align*}
        \Ep_{f^*} \left[ \|\hat{f} - f^*\|_{L^2(P_X)}^2\right] < \Ep_{f^*} \left[ \|\hat{f}^{\mathrm{lin}} - f^*\|_{L^2(P_X)}^2\right].
    \end{align*}
\end{corollary}
This result shows that a wide range of the other methods has larger generalization errors, hence the estimators by DNNs can overcome the other methods.
Some specific methods are analyzed in the supplementary material.

\subsection{Interpretation and Discussion}

According to the results, we can see that the estimators by DNNs have the theoretical advantage than the others for estimating $f^* \in \mF_{M, J,\alpha,\beta}$, since the estimators by DNNs achieve the optimal convergence rate of generalization errors and the others do not.

We provide an intuition on why DNNs are optimal and the others are not. 
The most notable fact is that DNNs can approximate non-smooth functions with a small number of parameters, due to activation functions and multi-layer structures.
A combination of two ReLU functions can approximate step functions, and a composition of the step functions in a combination of other parts of the network can easily express smooth functions restricted to pieces.
In contrast, even though the other methods have the universal approximation property, they require a larger number of parameters to approximate non-smooth structures.
By the statistical theory, a larger number of parameters increases the variance of estimators and worsens the performance, hence the other methods lose the optimality.

About the inefficiency of the other methods, we do not claim that every statistical method except DNNs misses the optimality for estimating piecewise smooth functions.
Our argument is the advantage of DNNs against linear estimators. 

Several studies investigate approximation and estimation for non-smooth structures.
Harmonic analysis provides several methods for non-smooth structures, such as curvelets \cite{candes2002recovering,candes2004new} and shearlets \cite{kutyniok2011compactly}.
While the studies provide an optimality for piecewise smooth functions on pieces with $C^2$ boundaries, pieces in the boundary fragment class considered in our study is more general and the harmonic-based methods cannot be optimal with the pieces (studied in \cite{korostelev2012minimax}).
Also, a convergence rate of generalization error is not known for these methods.
Studies from nonparametric statistics investigated non-smooth estimation \cite{van1985mean, wu1993kernel, wu1993nonparametric, wolpert2011stochastic,imaizumi2018statistically}.
These works focus on different settings such as density estimation or univariate data analysis, hence their setting does not fit problems discussed here.

\section{Experiments}


\subsection{Non-smooth Realization by DNNs} \label{sec:exp1}

We show how the estimators by DNNs can estimate non-smooth functions.
To this end, we consider the following data generating process with a piecewise linear function.
Let $D=2$, $\xi$ be an independent Gaussian variable with a scale $\sigma=0.5$, and $X$ be a uniform random variable on $I^2$.
Then, we generate $n$ pairs of $(X,Y)$ from \eqref{eq:reg} with a true function $f^*$ as piecewise smooth function such that
\begin{align}
    f^{*}(x) &= \mone_{R_1}(x)(0.2 + x_1^2 + 0.1 x_2) + \mone_{R_2}(x)(0.7 + 0.01  |4 x_1 + 10 x_2 - 9|^{1.5}), \label{def:exp_true}
\end{align}
with a set $R_1 = \{(x_1,x_2) \in I^2: x_2 \geq -0.6 x_1 + 0.75\}$ and $R_1 = I^2 \backslash R_1$.
A plot of $f$ in figure \ref{fig:true_f} shows its non-smooth structure.

\begin{center}
\begin{minipage}{0.4\hsize}
\begin{figure}[H]
\includegraphics[width=\hsize]{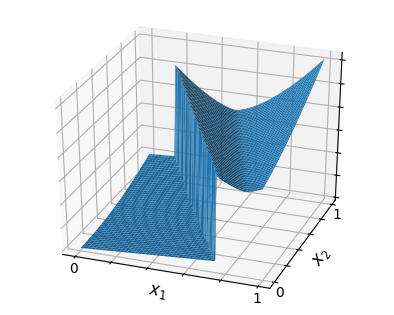}
\caption{A plot for $f^{*}(x_1,x_2)$ with $(x_1,x_2) \in I^2$. \label{fig:true_f}}
\end{figure}
\end{minipage}
\begin{minipage}{0.4\hsize}
\begin{figure}[H]
\includegraphics[width=\hsize]{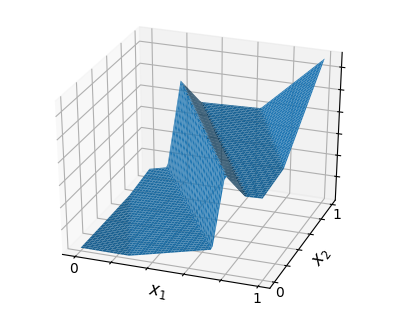}
\caption{A plot for the estimator $\hat{f}^{L}$. \label{fig:est_f}}
\end{figure}
\end{minipage}
\end{center}

\begin{center}
\begin{minipage}{0.52\hsize}
\begin{figure}[H]
\begin{center}
\includegraphics[width=0.8\hsize]{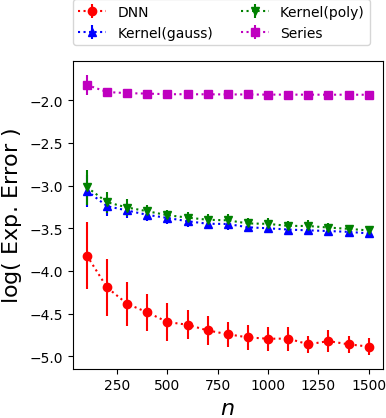}
\caption{Comparison of errors by the methods. \label{fig:compare}}
\end{center}
\end{figure}
\end{minipage}
\end{center}

About the estimation by DNNs, we employ the least square estimator \eqref{opt:erm}.
For the architecture $\Theta$ of DNNs, we set $|\Theta|=4$ and dimensionality of each of the layers as $D_1=2,D_\ell=3$ for $\ell \in \{2,3,4\}$, and $D_{5}=1$.
We use a ReLU activation.
To mitigate an effect of the non-convex optimization problem, we employ $100$ initial points which are generated from the Gaussian distribution with an adjusted mean.
We employ Adam \cite{kingma14adam} for optimization.

We generate data with a sample size $n=100$ f and obtain the least square estimator $\hat{f}^{L}$ for $f^{*}$.
Then, we plot $\hat{f}^{L}$ in Figure \ref{fig:est_f} which minimize an error from the $100$ trials with different initial points.
We can observe that $\hat{f}^{L}$ succeeds in approximating the non-smooth structure of $f^{*}$.

\subsection{Comparison with the Other Methods}

We compare performances of the estimator by DNNs, the orthogonal series method, and the kernel methods.
About the estimator by DNNs, we inherit the setting in Section \ref{sec:exp1}.
About the kernel methods, we employ estimators by the Gaussian kernel and the polynomial kernel.
A bandwidth of the Gaussian kernel is selected from $\{0.01, 0.1, 0.2,...,2.0\}$ and a degree of the polynomial kernel is selected from $[5]$.
Regularization coefficients of the estimators are selected from $\{0.01,0.4,0.8,...,2.0\}$.
About the orthogonal series method, we employ the trigonometric basis which is a variation of the Fourier basis.
All of the parameters are selected by a cross-validation.

We generate data from the process \eqref{eq:reg} with \eqref{def:exp_true} with a sample size $n \in \{100,200,...,1500\}$ and measure the expected loss of the methods.
In figure \ref{fig:compare}, we report a mean and standard deviation of a logarithm of the loss by $100$ replications.
By the result, the estimator by DNNs always outperforms the other estimators.
The other methods cannot estimate the non-smooth structure of $f^{*}$, although some of the other methods have the universal approximation property.

\section{Conclusion}

In this paper, we have derived theoretical results that explain why DNNs outperform other methods.
To this goal, we considered a regression problem under the situation where the true function is piecewise smooth.  We focused on the least square and Bayes estimators, and derived a convergence rate of generalization errors by the estimators.  Notably, we showed that the rates are optimal in the minimax sense.
Furthermore, we show that the commonly used linear estimators are inefficient to estimate piecewise smooth functions, hence we show that the estimators by DNNs work better than the other methods for non-smooth functions.
We also provided a guideline for selecting a number of layers and parameters of DNNs based on the theoretical results.


\section*{Acknowledgement}

We have greatly benefited from insightful comments and suggestions by Alexandre Tsybakov, Taiji Suzuki, Bharath K Sriperumbudur, Johannes Schmidt-Hieber, and Motonobu Kanagawa.

\bibliographystyle{plain}

\bibliography{./bib_master}

\onecolumn
\allowdisplaybreaks

\appendix

\begin{center}
    {\Large {\bf 
        Supplementary materials for \\
        ``Deep Neural Network Learn \\
        Non-Smooth Functions Effectively.''
    }}
\end{center}

\section{Additional definitions} \label{sec:holder}

\textbf{The H\"older Space}

Let $\Omega$ be an open subset of $\R^D$ and $\beta > 0$ a constant.  The H\"older space $H^\beta(\Bar{\Omega})$, where $\bar{\Omega}$ is the closure of $\Omega$, is the set of functions $f:\bar{\Omega}\to\R$ such that 
$f$ is continuously differentiable on $\bar{\Omega}$ up to the order $\lfloor\beta\rfloor$, and the $\lfloor\beta\rfloor$-the derivatives of $f$ are H\"older continuous with exponent ${\beta - \lfloor \beta \rfloor}$, namely, 
\[
\sup_{x,x' \in \bar{\Omega}, x \neq x'} \frac{|\partial^a f(x) - \partial^a f(x')|}{|x-x'|^{\beta - \lfloor \beta \rfloor}} <\infty
\]
for any multi-index $a$ with $|a|=\lfloor\beta\rfloor$, where $\partial^a$ denotes a partial derivative.. 
The norm of the H\"older space is defined by 
\begin{align*}
    \|f\|_{H^{\beta}} :=& \max_{|a| \leq \lfloor \beta \rfloor} \sup_{x \in \Omega} |\partial^a f(x)| + \max_{|a| = \lfloor \beta \rfloor} \sup_{x,x' \in \Omega, x \neq x'} \frac{|\partial^a f(x) - \partial^a f(x')|}{|x-x'|^{\beta - \lfloor \beta \rfloor}}.
\end{align*}

\textbf{Basis Pieces defined by Continuous Embeddings}

We redefine a piece as an intersection of $J$ embeddings of $D$-dimensional balls.  
we first introduce an extended notion of a \textit{boundary fragment class} which is developed by \cite{dudley1974metric} and \cite{mammen1999smooth}.

Preliminarily, let $\mS^{D-1} := \{x \in \R^D : \|x\|_2=1\}$ is the $D-1$ dimensional sphere, and let $(V_j,F_j)_{j=1}^\ell$ be its coordinate system as a $C^\infty$-differentiable manifold such that $F_j: V_j\to \mathring{B}^{D-1}:=\{x\in \R^{D-1}\mid \|x\|<1\}$ is a diffeomorphism.  A function $g:\mS^{D-1}\to\R$ is said to be in the H\"older class $H^\alpha(\mS^{D-1})$ with $\alpha>0$ if $g\circ F_j^{-1}$ is in $H^\alpha(\mathring{B})$. 

Let $B^D=\{x\in \R^D\mid \|x\|\leq 1\}$.  
A subset $R\subset I^D$ is called a {\em basic piece} if it satisfies two conditions: (i) there is a continuous embedding $g:B^D\to \R^D$ such that its restriction to the boundary $\mS^{D-1}$ is in $H^\alpha(\mS^{D-1})$  and $R=I^D\cap \mathrm{Image}(g)$, (ii) there is $1\leq i\leq D$ and  $h\in H^\alpha(I^{D-1})$ such that the indicator function of $R$ is given by the graph
\[
    1_{R} = \Psi(x_1,\ldots,x_{i-1}, x_i +  h(x_1,\ldots,\check{x}_i,\ldots,x_D),x_{i+1},...,x_D),
\]
where $\Psi$ is the Heaviside function. 
The condition (i) tells that a basic piece belongs to the \textit{boundary fragment class} which is developed by \cite{dudley1974metric} and \cite{mammen1999smooth}, while (ii) means $R$ is a set defined by a horizon function discussed in \cite{petersen2017optimal}.  

\section{Proof of Theorem \ref{thm:non-bayes}}

We first provide additional notations.
$\lambda$ denotes the Lebesgue measure.
For a function $f:I^D \to R$, $\|f\|_{L^\infty} = \sup_{x \in I^D}|f(x)|$ is a supremum norm.
$\|f\|_{L^2} := \|f\|_{L^2(I^D;\lambda)}$ is an abbreviation for $L^2(I^d;\lambda)$-norm.

Given a set of observations $\{X_1,\ldots,X_n\}$, let $\|\cdot\|_n$ be an empirical norm defined by 
\begin{align*}
    \|f\|_n^2 = n^{-1} \sum_{i=1}^n f(X_i)^2.
\end{align*}
The empirical norm of a random variable is also defined by  
\begin{align*}
    \|Y\|_n := \left(  n^{-1} \sum_{i \in [n]} Y_i^2\right)^{1/2} \mbox{~and~~}\|\xi\|_n := \left(  n^{-1} \sum_{i \in [n]} \xi_i^2\right)^{1/2}.
\end{align*}
The empirical norms are in fact seminorms, which do not satisfy the strong positivity. 

Let $\mF$ be a vector space with a norm $\|\cdot\|$.  For $\epsilon > 0$, the covering number of $\mF$ with $\|\cdot\|$ for radius $\delta$ is defined by 
\begin{align*}
    \mN(\epsilon, \mF, \|\cdot\|) := \inf\Bigl\{ N \mid \text{there is }\{f_j\}_{j \in [N]}\subset \mF \text{ such that } \|f - f_j\| \leq \epsilon, \forall f \in \mF \Bigr\}.
\end{align*}

By the definition of the least square estimator \eqref{opt:erm}, we obtain the following basic inequality 
\begin{align*}
    \|Y - \hat{f}^{L}\|_n^2 \leq \|Y - f\|_n^2
\end{align*}
for all $f \in \mF_{NN,\eta}(S,B,L)$.
It follows from $Y_i = f^*(X_i) + \xi_i$ that
\begin{align*}
    \|f^* +\xi - \hat{f}^{L}\|_n^2 \leq \|f^* + \xi - f\|_n^2.
\end{align*}
A simple calculation yields
\begin{align}
    \|f^* - \hat{f}^{L}\|_n^2 \leq \|f^* - f\|_n^2 + \frac{2}{n}\sum_{i=1}^n \xi_i (\hat{f}^{L}(X_i) - f(X_i)). \label{ineq:basic2}
\end{align}

In the following, we will fix $f \in \mF_{NN,\eta}(S,B,L)$ and evaluate each of the three terms in the RHS of \eqref{ineq:basic2}. 
In the first subsection, we provide a result for approximating $f^* \in \mF_{M,J,\alpha,\beta}$ by DNNs.
In the second subsection, we evaluate the variance of $\hat{f}^F$.
In the last subsection, we combine the results and derive an overall rate. 

\subsection{Approximate piecewise functions by DNNs} \label{sec:approx}

The purpose of this part is to bound the following error
\begin{align*}
    \|f - f^*\|_{L^2(P_X)}
\end{align*}
for properly selected $f \in \mF_{NN,\eta}(S,B,L)$.
To this end, we consider an existing $\Theta$ with properly selected $S, B$ and $L$.
Our proof is obtained by extending techniques by \cite{yarotsky2017error} and \cite{petersen2017optimal}.

Fix $f^* \in \mF_{M,J,\alpha,\beta}$ such that $f^* = \sum_{m \in [M]} f^*_m \mone_{R^*_m}$ with $f^*_m\in H^\beta$ and $R^*_m\in \mathcal{R}_{\alpha,J}$ for $m \in [M]$.
To approximate $f^*$, we introduce neural networks $\Theta_{f,m}$ and $\Theta_{r,m}$ for each $m \in [M]$, where the  number of layers $L$ and non-zero parameters $S$ will be specified later.

For approximation, we introduce some specific architectures of DNNs as building blocks.
The DNN $\Theta_+ := (A,b) = ((1,...,1)^\top,0)$ works as {\em summation}: $G_\eta[\Theta_+](x_1,...,x_D) = \sum_{d \in [D]} x_d$, and the DNN $\Theta_\times$ plays a role for multiplication: $G_\eta[\Theta_\times](x_1,...,x_D) \approx \prod_{d \in [D]}x_d$.
A network $\Theta_3$ approximates the inner product, i.e., $G_{\eta}[\Theta_3](x_1,...,x_M,x'_1,...,x'_M) \approx \sum_{m \in [M]} x_m x'_m$.
The existence and their approximation errors of $G_\eta[\Theta_\times]$ and $G_\eta[\Theta_3]$ will be shown in Lemma \ref{lem:multi_multi} and \ref{lem:sum_multi}.

We construct a network given by $G_{\eta}[\Theta_3](G_{\eta}[\Theta_1](\cdot), G_{\eta}[\Theta_2](\cdot))$, where $\Theta_1$ and $\Theta_2$ consist of $M$-dimensional outputs $(\Theta_{f,1},\ldots,\Theta_{f,M})$ and $(\Theta_{r,1},\ldots,\Theta_{r,M})$, respectively.  
We evaluate the distance between $f^*$ and the combined neural network :
\begin{align}
    &\|f^* - G_{\eta}[\Theta_3](G_{\eta}[\Theta_{f,1}](\cdot),...,G_{\eta}[\Theta_{f,M}](\cdot), G_{\eta}[\Theta_{r,1}](\cdot),...,G_{\eta}[\Theta_{r,M}](\cdot))\|_{L^2} \notag \\
    &= \left\| \sum_{m \in [M] }f^*_m \mone_{R^*_m} - G_{\eta}[\Theta_3](G_{\eta}[\Theta_{f,1}](\cdot),...,G_{\eta}[\Theta_{f,M}](\cdot), G_{\eta}[\Theta_{r,1}](\cdot),...,G_{\eta}[\Theta_{r,M}](\cdot))\right\|_{L^2} \notag \\
    & \leq  \left\| \sum_{m \in [M] }f^*_m \otimes \mone_{R^*_m} - \sum_{m \in [M] }G_{\eta}[\Theta_{f,m}] \otimes  G_{\eta}[\Theta_{r,m}]\right\|_{L^2} \notag \\
    & \quad +  \left\| \sum_{m \in [M] }G_{\eta}[\Theta_{f,m}] \otimes  G_{\eta}[\Theta_{r,m}] -  G_{\eta}[\Theta_3](G_{\eta}[\Theta_{f,1}](\cdot),...,G_{\eta}[\Theta_{f,M}](\cdot), G_{\eta}[\Theta_{r,1}](\cdot),...,G_{\eta}[\Theta_{r,M}](\cdot))\right\|_{L^2}\notag \\
    & \leq \sum_{m \in [M] } \left\| f^*_m \otimes \mone_{R^*_m} - G_{\eta}[\Theta_{f,m}] \otimes  G_{\eta}[\Theta_{r,m}]\right\|_{L^2} \notag \\
    &\quad +  \left\| \sum_{m \in [M] }G_{\eta}[\Theta_{f,m}] \otimes  G_{\eta}[\Theta_{r,m}] -  G_{\eta}[\Theta_3](G_{\eta}[\Theta_{f,1}](\cdot),...,G_{\eta}[\Theta_{f,M}](\cdot), G_{\eta}[\Theta_{r,1}](\cdot),...,G_{\eta}[\Theta_{r,M}](\cdot))\right\|_{L^2} \notag \\
    & \leq \sum_{m \in [M] } \left\| ( f^*_m -  G_{\eta}[\Theta_{f,m}]) \otimes   G_{\eta}[\Theta_{r,m}]  \right\|_{L^2} +\sum_{m \in [M] }  \left\| f_m^*\otimes  (\mone_{R^*_m}  - G_{\eta}[\Theta_{r,m}])\right\|_{L^2} \notag \\
    &\quad +  \left\| \sum_{m \in [M] }G_{\eta}[\Theta_{f,m}] \otimes  G_{\eta}[\Theta_{r,m}] -  G_{\eta}[\Theta_3](G_{\eta}[\Theta_{f,1}](\cdot),...,G_{\eta}[\Theta_{f,M}](\cdot), G_{\eta}[\Theta_{r,1}](\cdot),...,G_{\eta}[\Theta_{r,M}](\cdot))\right\|_{L^2} \notag \\
    & =: \sum_{m \in [M]} B_{1,m} + \sum_{m \in [M]} B_{2,m} + B_3. \label{ineq:decomp1}
\end{align}
We will bound $B_{m,1},B_{m,2}$ for $m \in [M]$ and $B_3$.

\paragraph{Bound of $B_{1,m}$.}
The H\"older inequality gives 
\begin{align*}
    &\left\| ( f^*_m -  G_{\eta}[\Theta_{f,m}]) \otimes  \mone_{R^*_m}  \right\|_{L^2} \leq \left\| f^*_m -  G_{\eta}[\Theta_{f,m}] \right\|_{L^2}\left\|   G_{\eta}[\Theta_{r,m}] \right\|_{L^\infty}.
\end{align*}
Theorem 1 in \cite{yarotsky2017error} and Theorem A.9 in \cite{petersen2017optimal} guarantee that there exists a neural network $\Theta_{f,m}$ such that $|\Theta_{f,m}|\leq c_1'(1+log_2\lceil(1+\beta)\rceil\cdot (1+\beta/D)$, $\|\Theta_{f,m}\|_0 \leq C'_1 \epsilon^{-D/\beta}$0, $\|\Theta_{f,m}\|_{\infty} \leq \epsilon^{-2s_1}$,and $\left\| f^*_m -  G_{\eta}[\Theta_{f,m}] \right\|_{L^2} < \epsilon$, where $c_1,c_1',s_1>0$ are constants depending only on $f^*$.  
The neural network $ \Theta_{r,m}$ is given by Lemma 3.4 in \cite{petersen2017optimal}, for which  $\|G_{\eta}[\Theta_{r,m}]\|_{L^\infty} \leq 1$.
Combining these results, we obtain
\begin{align*}
    B_{1,m} < \epsilon.
\end{align*}

\paragraph{Bound of $B_{2,m}$.} We have
\begin{align*}
    &\left\| f_m^*\otimes  (\mone_{R^*_m}  - G_{\eta}[\Theta_{r,m}]) \right\|_{L^2} \leq \left\|   f_m^*\right\|_{L^\infty}\left\| \mone_{R^*_m}  - G_{\eta}[\Theta_{r,m}] \right\|_{L^2}.
\end{align*}
From $f_m^* \in H^\beta(I^D)$, there exists a constant $C_H > 0$ such that $\|f_m^*\|_{L^2} \leq C_H$.

Recall that each $R_m^*\in \mathcal{R}_{\alpha,J}$ takes the form $R_m^*=\cap_{j=1}^J R_m^j$ with $R_m^j\in \mathcal{R}_{\alpha,1}$ for some $B$, and thus $1_{R_m^j}\in \mathcal{HF}_{\alpha,D,B}$ defined in \cite{petersen2017optimal}.
Then, from Lemma 3.4 in \cite{petersen2017optimal}, there are some constants $c'$, $c$, and $s>0$ depending on $\alpha$, $D$, and $B$ such that for any $\varepsilon>0$ a neural network $\Theta_{m,j}$ can be found with 
\[
	\Vert 1_{R_{m}^{j}}-G_\eta[\Theta_{m,j}]\Vert_{L^2}\leq \varepsilon, 
\]
$|\Theta_{m,j}|\leq c'_2 (1+\alpha/D) \log(2+\alpha)$, $\Vert \Theta_{m,j}\Vert_0 \leq c_2 \varepsilon^{-2(D-1)/\alpha}$, and $\Vert \Theta_{m,j} \Vert_\infty \leq \varepsilon^{-2s_2}$.  Note that $c_2,c_2',s_2$ depend only on $f^*$ for our purpose. 

Define a neural network $\Theta_{r,m}$ by 
\[
G[\Theta_{r,m}]:=G[\Theta_{\times,J}](G_\eta[\Theta_{m,1}](\cdot),\ldots,G_\eta[\Theta_{m,J}](\cdot)),
\]
where $\Theta_{\times,J}$ is given in Lemma \ref{lem:multi_multi} below.  It follows that 
\begin{align}\label{eq:B2}
& \left\| \mone_{R^*_m}  - G_{\eta}[\Theta_{r,m}] \right\|_{L^2} \notag \\
& \leq \Vert \mone_{R^*_m} - \bigotimes_{j \in [J]} G_\eta[\Theta_{m,j}]\Vert_{L^2} +
\Vert \bigotimes_{j \in [J]} G_\eta[\Theta_{m,j}] - G_{\eta}[\Theta_{r,m}]\Vert_{L^2}
\end{align}
The first term of the last line of \eqref{eq:B2} is bounded by 
\begin{align*}
    &\Vert \bigotimes_{j \in [J]}1_{R_{m,j}} -  \bigotimes_{j \in [J]}G_\eta[\Theta_{m,j}]\Vert_{L^2} \\
    & =\sum_{j\in[J]} \| 1_{R_{m,j}} -  G_\eta[\Theta_{m,j}]\|_{L^2}  \prod_{j'=1 }^{j} \|1_{R_{m,j'}}\|_{L^2} \prod_{j''\in [J] \backslash [j]}\| G_{\eta}[\Theta_{m,j''}]\Vert_{L^2}\\
    & \leq \sum_{j \in[J]} \Vert 1_{R_{m,j}}  -  G_{\eta}[\Theta_{m,j}]\Vert_{L^2},
\end{align*}
where $\Vert G_{\eta}[\Theta_{r,m}] \Vert_\infty\leq 1$ is used in the last line.  From Lemma \ref{lem:multi_multi}, the second term of \eqref{eq:B2} is upper bounded by $(J-1)\varepsilon$.  
We finally obtain 
\[
    B_{2,m} := \left\| f_m^*\otimes  (\mone_{R^*_m}  - G_{\eta}[\Theta_{r,m}])\right\|_{L^2} 
\leq C_H (2J-1) \varepsilon.
\]

\begin{lemma}\label{lem:multi_multi}
    Fix $\theta > 0$ arbitrary.
There are absolute constants $C_\times>0$ and $s_\times >0$ such that for any $\epsilon \in (0,1/2)$, $D'\in\mathbb{N}$ there exists a neural network $\Theta_{\times,D'}$ of $D'$-dimensional input with at most $(1 + \log_2D') / \theta$ layers, $\|\Theta_{\times,D'}\|_0 \leq C_{\times}D' \epsilon^{-\theta}$, $\|\Theta_{2}'\|_\infty \leq \epsilon^{-2s}$, and 
    \begin{align*}
        \Vert \prod_{d \in [D']}x_d - G_{\eta}[\Theta_{\times,D'}](x_1,...,x_{D'}) \Vert_{L^\infty([-1,1]^{D'})} \leq (D'-1) \epsilon.
    \end{align*}
\end{lemma}
\begin{proof}
We employ the neural network for multiplication $\Theta_{\times,D'}$ as Proposition 3 in \cite{yarotsky2017error} and Lemma A.3 in \cite{petersen2017optimal}, and consider a tree-shaped multiplication network.     There are $D'-1$ multiplication networks and the tree has $1 + \log_2D'$ depth.
\end{proof}

\paragraph{Bound of $B_{3}$. } 
Take $\Theta_3$ as the neural network in Lemma \ref{lem:sum_multi}. Then we obtain 
\begin{align*}
    B_3 \leq M\epsilon.
\end{align*}

\begin{lemma} \label{lem:sum_multi}
Let $\theta > 0$ be arbitrary.
Then, with the constants $C_{\times},s>0$ in Lemma \ref{lem:multi_multi}, for each $\epsilon \in (0,1/2)$ and $D'\in\mathbb{N}$, there exists a neural network $\Theta_3$ for a $2D'$-dimensional input with at most $1 + L$ layers where $L>1/\theta$ and $D' +  C_\times D'  \epsilon^{-\theta}$ non-zero parameters such that $\|\Theta_3\|_\infty \leq \epsilon^{-s}$ and 
    \begin{align*}
        \left| G_{\eta}[\Theta_3](x_1,\ldots,x_{D'},X_{D'+1},\ldots,x_{2D'}) - \sum_{d \in [D']} x_{d} x_{D' + d} \right| \leq D' \epsilon.
    \end{align*}
\end{lemma}
\begin{proof}
Let $\Theta_3$ be a neural network defined by 
\begin{align*}
    G_{\eta}[\Theta_3](x) = G_{\eta}[\Theta_+](G_{\eta}[\Theta_\times](x_1,x_{D' + 1}),...,G_{\eta}[\Theta_\times](x_{D'},x_{2D'})),
\end{align*}
where $\Theta_\times$ is given by Lemma \ref{lem:multi_multi}, and $\Theta_+$ is the sammation network given by
\[
\Theta_+ := (A,b)=( (1,\ldots,1)\top,0).
\]
Then, we evaluate the difference as
\begin{align*}
    &\left| G_{\eta}[\Theta_3](x_1,...,x_{2D'}) - \sum_{d \in [D']} x_d x_{2d} \right| 
= \left| \sum_{d \in [D']}G_{\eta}[\Theta_\times](x_{d},x_{D' + d})  - \sum_{d \in [D']} x_d x_{D' + d} \right| \\
    &\leq \sum_{d \in [D']} \left| G_{\eta}[\Theta_\times](x_{d},x_{D' + d})  - x_d x_{D' + d} \right| \leq D' \epsilon, 
\end{align*}
where the last inequality uses Lemma \ref{lem:multi_multi}.
\end{proof}

\paragraph{Combined bound}
We combine the results about $B_{1,m},B_{2,m}$ and $B_3$, then define $\dot{f} \in \mF_{NN,\eta}(S,B,L)$ for approximating $f^*$.  

For $\Theta_1$
\[
|\Theta_1|\leq c_1'\bigl(1+\lceil\log_2(1+\beta)\rceil\cdot(1+\beta/D)\bigr), \quad
\vert\Theta_1\vert_0 \leq  Mc_1\varepsilon_1^{-D/\beta},\quad 
\Vert\Theta_1\Vert_\infty \leq \varepsilon^{-2s_1}.
\]
For $\Theta_2$,
\[
|\Theta_2|\leq c_2'\bigl(1+\lceil\log_2(2+\alpha)\rceil\cdot(1+\alpha/D)\bigr) + \frac{1+\log_2 J}{\theta_2}, \quad
\vert\Theta_2\vert_0 \leq  MJ \bigl( c_2\varepsilon_2^{-(2D-2)/\alpha} + c_\times \varepsilon_2^{-\theta_2}\bigr),\quad 
\Vert\Theta_2\Vert_\infty \leq \varepsilon^{-2s_2}.
\]
For $\Theta_3$, 
\[
|\Theta_3|\leq  1+ 1/\theta_3, \quad
\vert\Theta_3\vert_0 \leq M + c_\times M \varepsilon^{-\theta_3},\quad 
\Vert\Theta_3\Vert_\infty \leq \varepsilon_3^{-2s_3}.
\]

To balance the approximation error and estimation error, the latter of which will be discussed later, we choose the $\varepsilon_i$ ($i=1,2,3$) and $\theta_i$ ($i=2,3$) as follows:
\begin{gather}
\varepsilon_1:=a_1 n^{-\beta/(2\beta+D)}, \quad \varepsilon_2:= a_2 n^{-\alpha/(2\alpha+2D-2)}, \quad
\varepsilon_3:= a_3 \max\{-\beta/(2\beta+D),n^{-\alpha/(2\alpha+2D-2)} \}, \\
\theta_2:= (2D-2)/\alpha, \qquad \theta_3:= \min\{(2D-2)/\alpha, D/\beta\}, \label{select:param}
\end{gather}
where $a_1,a_2,a_3$ are arbitrary positive constants.  

The total network $\dot{\Theta}$ to give $\dot{f}:=G_\eta[\Theta_3](G_\eta[\Theta_1],G_\eta[\Theta_2])$.
With the above choice of $\varepsilon_i$ and $\theta_i$, the maximum numbers of layers, non-zero parameters, and maximum absolute value of parameters in  are bounded by 
\begin{align*}
& |\Theta| \leq C_L (1+\log_2\bigl( \max\{1+\beta,2+\alpha,1+\log_2 J\}\bigr) )(1+\max\{\beta/D,\alpha/(2D-2)\}), \\ 
& \Vert \Theta\Vert_0 \leq  Mc_1\varepsilon_1^{-D/\beta} + MJ \bigl( c_2\varepsilon_2^{-(2D-2)/\alpha}\bigr)  + M + c_\times M \varepsilon^{-\theta_3} \\
& \quad \leq C_S M \Bigl\{ 1 + J\max\{ n^{D/(2\beta+D)}, n^{2(D-1)/(2\alpha+2D-2)}\}\Bigr\},\\
& \Vert \Theta\Vert_\infty \leq C_B \max\{n^{2s (2\beta+D)/\beta},n^{2s (2\alpha+2D-2)/\alpha} \},
\end{align*}
where $s>0$ is a positive constant depending only on $f^*$.  The approximation error is given by
\begin{align} 
& \Vert f^* - \dot{f}\Vert_{L^2} \notag \\
&  \leq  a_1' M n^{-\beta/(2\beta+D)} + C_H a_2' M (2J-1) n^{-\alpha/(2\alpha+2D-2)} + M \max\{ n^{-\beta/(2\beta+D)}, n^{-\alpha/(2\alpha+2D-2)}\} \notag \\
& \leq  C_{apr} (2J+1)M \max\{ n^{-\beta/(2\beta+D)}, n^{-\alpha/(2\alpha+2D-2)}\},\label{ineq:approx_final}
\end{align}
where $C_{apr}>0$ is a constant.


\subsection{Evaluate an entropy bound of the estimators by DNNs} \label{sec:entropy}

Here, we evaluate a variance term of $\|\hat{f}^{L} - f^*\|_n$ in \eqref{ineq:basic2} through evaluating the term
\begin{align*}
    \left|\frac{2}{n}\sum_{i \in [n] } \xi_i (\hat{f}^{L}(X_i) - f(X_i)) \right|.
\end{align*}
To bound the term, we employ the technique by the empirical process technique \cite{koltchinskii2006local, gine2015mathematical, suzuki2017fast}.

We consider an expectation of the term.
Let us define a subset $\tilde{\mF}_{NN,\delta} \subset \mF_{NN,\eta}(S,B,L)$ by 
\[
\tilde{\mF}_{NN,\delta} := \{f - \hat{f}^{L} : \|f - \hat{f}^{L}\|_n \leq \delta, f \in \mF_{NN,\eta}(S,B,L)\}.
\]
Here, we mention that $f \in \tilde{\mF}_{NN,\delta} $ is bounded by providing the following lemma.
\begin{lemma}\label{lem:bounded}
    For any $f \in \mF_{NN,\eta}(S,B,L)$ with an activation function $\eta$ satisfying Lipschitz continuity with a constant $1$, we obtain
    \begin{align*}
        \|f\|_{L^\infty} \leq B_\mF,
    \end{align*}
    where $B_\mF > 0$ is a finite constant.
\end{lemma}
\begin{proof}
For each $\ell \in [L]$, consider a transformation
\begin{align*}
    f_\ell(x) := \eta( A_\ell x + b_\ell).
\end{align*}
When $\|x\|_\infty = B_x$ and $\|\vect(A_\ell)\|_\infty, \|b_\ell\|_\infty \leq B$, we obtain 
\begin{align*}
    \|f_\ell\|_{L^\infty} \leq \|A_\ell x + b_\ell\|_\infty \leq D_\ell B_x B + B.
\end{align*}
Let $\bar{D} := \max_{\ell \in [L]} D_\ell$, when iteratively we have
\begin{align*}
    \|f\|_{L^\infty} \leq \sum_{\ell \in [L] \cup \{0\}} \prod_{\ell' \in [L] \backslash [\ell]} (\bar{D}B)^{\ell'}  < \infty,
\end{align*}
by applying that $\|x\|_\infty \leq 1$ for an input.
\end{proof}

Due to Lemma \ref{lem:bounded}, with given $\{X_i\}_{i \in [n]}$, we can apply the chaining (Theorem 2.3.6 in \cite{gine2015mathematical}) and obtain
\begin{align*}
    2\Ep_{\xi} \left[ \sup_{f' \in \tilde{F}_{NN,\delta} }\left| \frac{1}{n}\sum_{i \in [n]}\xi_i f'(X_i) \right|\right] \leq 8 \sqrt{2} \frac{\sigma}{n^{1/2}} \int_{0}^{\delta/2} \sqrt{\log 2 \mN(\epsilon', \mF_{NN,\eta}(S,B,L), \|\cdot\|_n)} d\epsilon'.
\end{align*}
Here, to apply Theorem 2.3.6 in \cite{gine2015mathematical}, we set $n^{-1/2} \sum_{i \in [n]} \xi_i f(X_i)$ as the stochastic process and $0$ as $X(t_0)$ in the theorem.
Then, to bound the entropy term, we apply an inequality
\begin{align*}
        \log \mN(\epsilon, \mF_{NN,\eta}(S,B,L), \|\cdot\|_{n}) &\leq \log \mN(\epsilon, \mF_{NN,\eta}(S,B,L), \|\cdot\|_{L^\infty}) \\
        &\leq (S+1) \log \left( \frac{2(L+1)N^2}{B\epsilon} \right),
    \end{align*}
the last inequality holds by Theorem 14.5 in \cite{anthony2009neural} and Lemma 8 in \cite{schmidt2017nonparametric}, and the constant $N$ is defined by
\[
N := ...
\]
Then, we obtain
\begin{align}
    2\Ep_{\xi} \left[ \sup_{f' \in \tilde{F}_{NN,\delta} }\left| \frac{1}{n}\sum_{i\in [n]} \xi_i f'(X_i) \right|\right]  \leq 4 \sqrt{2} \frac{\sigma \sqrt{S+1}\delta}{n^{1/2}} \left(\log \frac{(L+1)N^2}{B\delta} + 1\right). \label{ineq:bound_var1}
\end{align}
With the bound \eqref{ineq:bound_var1} for the expectation term, we apply the Gaussian concentration inequality (Theorem 2.5.8 in \cite{gine2015mathematical})  by setting $n^{-1} \sum_{i \in [n]} \xi_i f'(X_i)$ as the stochastic process and $\delta^2\geq \|f\|_{n}^2$ be $B^2$  (in Theorem 2.5.8, \cite{gine2015mathematical}), and obtain
\begin{align}
    &1 - \exp(-nu^2 / 2 \sigma^2 \delta^2) \\
    &\leq \Pr_\xi \left( 4 \sup_{f' \in \tilde{F}_{NN,\delta} } \left| \frac{1}{n}\sum_{i \in [n]} \xi_i f'(X_i) \right|\leq 4\Ep_\xi \left[ \sup_{f' \in \tilde{F}_{NN,\delta} }\left| \frac{1}{n}\sum_{i=1}^n \xi_i f'(X_i) \right|\right] + u \right) \notag \\
    & \leq \Pr_\xi \left(4 \sup_{f' \in \tilde{F}_{NN,\delta} } \left| \frac{1}{n}\sum_{i \in [n]} \xi_i f'(X_i) \right| \leq 8 \sqrt{2} \frac{\sigma \sqrt{S+1}\delta}{n^{1/2}} \left(\log \frac{(L+1)N^2}{B\delta} + 1\right) + u \right), \label{ineq:bound_var2}
\end{align}
for any $u > 0$.
Let us introduce the following notation for simplicity:
\begin{align*}
    V_n :=  8 \sqrt{2} \frac{\sigma \sqrt{S+1}}{n^{1/2}}.
\end{align*}


To evaluate the variance term, we reform the basic inequality \eqref{ineq:basic2} as
\begin{align*}
    - \frac{2}{n}\sum_{i=1}^n \xi_i (\hat{f}^{L}(X_i) - f(X_i)) + \|f^* - \hat{f}^{L}\|_n^2 \leq \|f^* - f\|_n^2,
\end{align*}
and apply an inequality $\frac{1}{2}\|\hat{f}^{L} - f\|_n^2 \leq \|f - f^* \|_n^2 + \| f^* - \hat{f}^{L}\|_n^2$, then we have
\begin{align*}
    - \frac{2}{n}\sum_{i=1}^n \xi_i (\hat{f}^{L}(X_i) - f(X_i)) +\frac{1}{2}\|\hat{f}^{L} - f\|_n^2 - \|f - f^* \|_n^2  \leq \|f^* - f\|_n^2,
\end{align*}
then we have
\begin{align}
    - \frac{2}{n}\sum_{i=1}^n \xi_i (\hat{f}^{L}(X_i) - f(X_i)) +\frac{1}{2}\|\hat{f}^{L} - f\|_n^2 \leq 2\|f^* - f\|_n^2.     \label{ineq:bound_var3}
\end{align}

Consider a lower bound for $ - \frac{2}{n}\sum_{i \in [n] } \xi_i (\hat{f}^{L}(X_i) - f(X_i)) $.
To make the bound \eqref{ineq:bound_var2} be valid for all $f \in \mF_{NN,\eta}(S,B,L)$, we let  $\delta = \max\{\|\hat{f}^{L} - f\|_n , V_n\}$.
Then, we obtain the bound
\begin{align*}
    &\left|\frac{2}{n}\sum_{i \in [n] } \xi_i (\hat{f}^{L}(X_i) - f(X_i)) \right| \\
    &\leq \max\{\|\hat{f}^{L} - f\|_n , V_n\} \left\{ V_n \left( \log \frac{(L+1)N^2}{BV_n}  + 1 \right) \right\} + u \\
    &\leq \frac{1}{4}\left( \max\{\|\hat{f}^{L} - f\|_n , V_n\} \right)^2 + 2\left\{ V_n \left( \log \frac{(L+1)N^2}{BV_n}  + 1 \right) \right\}^2 + u,
\end{align*}
by using $xy \leq \frac{1}{4}x^2 + 2 y^2$.
Using this result to \eqref{ineq:bound_var3}, we obtain
\begin{align*}
    &-\frac{1}{4}\left( \max\{\|\hat{f}^{L} - f\|_n , V_n\} \right)^2 - 2\left\{ V_n \left( \log \frac{(L+1)N^2}{BV_n}  + 1 \right) \right\}^2 - u +\frac{1}{2}\|\hat{f}^{L} - f\|_n^2 \\
    &\leq 2\|f^* - f\|_n^2.    
\end{align*}
If $\|\hat{f}^{L} - f\|_n \geq V_n $ holds, we obtain
\begin{align*}
    -\frac{1}{4}\|\hat{f}^{L} - f\|_n^2 - 2\left\{ V_n \left( \log \frac{(L+1)N^2}{BV_n}  + 1 \right) \right\}^2 - u +\frac{1}{2}\|\hat{f}^{L} - f\|_n^2 \leq 2\|f^* - f\|_n^2.    
\end{align*}
Then, simple calculation yields
\begin{align}
    \|\hat{f}^{L} - f\|_n^2 \leq 4 \left\{ V_n \left( \log \frac{(L+1)N^2}{BV_n}  + 1 \right) \right\}^2 + 2u + 4\|f^* - f\|_n^2.    \label{ineq:bound_var4}
\end{align}
If $\|\hat{f}^{L} - f\|_n \leq V_n $, the same result holds.

We additionally apply an inequality $\frac{1}{2}\|\hat{f}^{L} - f^*\|_{L^2}^2 \leq \|f^* - f\|_n^2 + \|\hat{f}^{L} - f\|_n^2$ to \eqref{ineq:bound_var4}, we obtain
\begin{align}
    \|\hat{f}^{L} - f^*\|_n^2 \leq 10\|f^* - f\|_n^2 + 8 \left\{ V_n \left( \log \frac{(L+1)N^2}{BV_n}  + 1 \right) \right\}^2 + 4u,    \label{ineq:bound_var5}
\end{align}
with probability at least $1 - \exp(-nu^2 / 2 \sigma^2 \delta^2)$ for all $u>0$.

\subsection{Combine the results}

Combining the results in Sections \ref{sec:approx} and \ref{sec:entropy}, we evaluate the expectation of the LHS of \eqref{ineq:bound_var5}, i.e. $\|\hat{f}^{L} - f^*\|_{L^2(P_X)}$.
To this end, we substitute $\dot{f}$ in Section \ref{sec:approx} into $f$ in \eqref{ineq:bound_var5}.  First, note 
\begin{align}
    &\Ep_X\left[ \|\dot{f} - f^*\|_n^2 \right] = \int_{[0,1]^D} (\dot{f} - f^*)^2 dP_X =  \int_{[0,1]^D} (\dot{f} - f^*)^2 d \lambda \frac{dP_X}{d\lambda} \notag \\
    &\leq \|\dot{f} - f^*\|_{L^2}^2 \sup_{x \in [0,1]^D}p_X(x) \label{ineq:bound_l2p}
\end{align}
by the H\"older's inequality.  Here, $p_X$ is a density of $P_X$ and $\sup_{x \in [0,1]^D}p_X(x) \leq B_P$ is finite by the assumption.
Also, it follows from Bernstein's inequality that for any $u>0$
\begin{equation}\label{eq:Bernstein}
	\Pr\Bigl( \| \dot{f}  -f^* \|_n^2 \leq \|\dot{f}-f^*\|_{L^2(P_X)}^2 + u \Bigr) \geq 1-\exp\Bigl( -\frac{nu^2}{A^2 s_n + Au}\Bigr),
\end{equation}
where $A$ is a constant with $\|\dot{f}\|_\infty \leq A$ and $\|f^*\|_\infty \leq A$, and $s_n = \Ep| \dot{f}(X)-f^*(X)|^2$.

In \eqref{ineq:bound_var5}, by the choice $f=\dot{f}$, we see that $\delta^2 \leq  C\max\{n^{-2\beta/(2\beta+D)}, n^{-2\alpha/(2\alpha+2D-2)}\}$ with some constant $C>0$, and thus $\exp(-nu^2/(2\sigma^2\delta^2))$ converges to zero for $u=C_u/n$ with a constant $C_n>0$.  Additionally, in \eqref{eq:Bernstein}, since $s_n\leq  C\max\{n^{-2\beta/(2\beta+D)}, n^{-2\alpha/(2\alpha+2D-2)}\}$ with some constant $C>0$, for $u=C_u/n$, we have $\exp( -nu^2/(A^2 s_n + Au))$ goes to zero.  
It follows then 
\begin{align*}
    &\|\hat{f}^{L} - f^*\|_{L^2(P_X)}^2 \\
    &\leq 10B_P \|\dot{f} - f^*\|_{L^2}^2 + 8 \left\{ V_n \left( \log \frac{(L+1)N^2}{BV_n}  + 1 \right) \right\}^2 + 5u\\
    & \leq C_{a}^2 (2J+1)^2M^2 \max\{n^{-2\beta/ (2\beta + D)},n^{-2\alpha / (2\alpha + 2D-2)}\} \\
    &\quad \quad + 128 \frac{\sigma^2 (S+1)}{n} \left(\log \frac{(L+1)N^2}{BV_n} + 1\right)^2 + \frac{5C_u}{n},
\end{align*}
with probability converging to one, where $C_a>0$ is a constant.
Using the bound of the number of non-zero parameters $S \leq  C_S M \Bigl\{ 1 + J\max\{ n^{D/(2\beta+D)}, n^{2(D-1)/(2\alpha+2D-2)}\Bigr\}$, 
we obtain 
\begin{align*}
    &\|\hat{f}^{L} - f^*\|_{L^2(P_X)}^2\\
    &\leq \left\{C_{a}^2 (2J+1)^2M^2  + 1024\sigma^2C_S M(1 +  J)\left(\log \frac{(L+1)N^2}{BV_n} + 1\right)^2\right\}\\
    & \quad \times \max\{n^{-2\beta/ (2\beta + D)},n^{-\alpha / (\alpha + D-1)}\} + \frac{1024\sigma^2 + 5 C_u}{n}.
\end{align*}
Since $V_n, V, L, B$ are polynomial to $n$, this completes the proof of Theorem \ref{thm:non-bayes}. 

\qed

\section{Proof of Theorem \ref{thm:bayes}}

We follow a technique developed by \cite{vaart2011information} and evaluate contraction of the posterior distribution.
To this end, we consider the following two steps.
At the first step, we consider a bound for the distribution with an empirical norm $\|\cdot\|_n$.
Secondly, we derive a bound with an expectation with respect to the $L^2(P_X)$ norm.

In this section, we reuse $\dot{f} \in \mF_{NN,\eta}(S,B,L)$ by the neural network $\dot{\Theta}$ which is defined in Section \ref{sec:approx}.
By employing $\dot{f}$, we can use the bounds for an approximation error  $\|f^* - \dot{f}\|_{L^2}$, a number of layers in $\dot{\Theta}$, and a number of non-zero parameters $\|\dot{\Theta}\|_0$.

\subsection{Bound with an empirical norm} \label{sec:bayes_fixed}

\textbf{Step 1. Preparation}

To evaluate the convergence, we provide some notions for preparation.

We use addition notation for the dataset $Y_{1:n} := (Y_1,...,Y_n)$ and $X_{1:n} := (X_1,...,X_n)$ and a probability distribution of $Y_{1:n}$ given $X_{1:n}$ such as
\begin{align*}
    P_{n,f} = \prod_{i \in [n]} \mN(f(X_i), \sigma^2),
\end{align*}
with some function $f$.
Let $p_{n,f}$ be a density function of $P_{n,f}$.

Firstly, we provide an event which characterizes a distribution of a likelihood ratio.
We apply Lemma 14 in \cite{vaart2011information} we obtain that
\begin{align*}
    &P_{n,f^*} \left( \int \frac{p_{n,f}(Y_{1:n})}{p_{n,f^*}(Y_{1:n})} d \Pi_f(f) \geq \exp(-r^2) \Pi_f (f : \|f - f^*\|_n < r) \right) \geq 1 - \exp(-nr^2/8),
\end{align*}
for any $f$ and $r > 0$.
To employ the entropy bound, we will update $ \Pi_f (f : \|f - f^*\|_n < r) $ of this bound as  $\Pi_f (f : \|f - \dot{f}\|_{L^\infty} < r) $.
To this end, we apply Lemma \ref{lem:talagrand} then it yields the following bound such for $\|f - f^*\|_n$ as
\begin{align*}
    1-\exp(-nr^2 / B_f^2) \leq \Pr_X \left( \|f - f^*\|_n \leq \|f - \dot{f}\|_{L^\infty} + B_p \|\dot{f} - f^*\|_{L^2} + r \right),
\end{align*}
for any $r$ and a parameter $B_f > 0$.
Using the inequality \eqref{ineq:approx_final} for $\|\dot{f} - f^*\|_{L^2}$, we define $\epsilon_n$ as
\begin{align*}
    \epsilon_n \geq  \|\dot{f} - f^*\|_{L^2},
\end{align*}
and also substitute $r = B_p \epsilon_n$, then we have
\begin{align*}
    1-\exp(-nB_p^2 \epsilon_n^2 / B_f^2) \leq \Pr_X \left( \|f - f^*\|_n \leq \|f - \dot{f}\|_{L^\infty} + 2B_p \epsilon_n \right).
\end{align*}
Then, we consider an event $\mE_r$ as follows and obtain that
\begin{align}
    P_{n,f^*} \left( \mE_r \right) &:= P_{n,f^*} \left( \int \frac{p_{n,f}(Y_{1:n})}{p_{n,f^*}(Y_{1:n})} d \Pi_f(f) \geq \exp(-r^2) \Pi_f (f :  \|f - \dot{f}\|_{L^\infty}  < B_p \epsilon_n) \right) \notag \\
    &\quad \geq 1 - \exp(-n9 B_p^2 \epsilon_n^2 /8) - \exp(-nB_p^2 \epsilon_n^2 / B_f^2), \label{ineq:event}
\end{align}
by substituting $r = 3B_p \epsilon_n$.

Secondly, we provide a test function $\phi: Y_{1:n} \mapsto z \in \R$ which can identify the distribution with $f^*$ asymptotically.
Let $\Ep_{n,f}[\cdot]$ be an expectation with respect to $P_{n,f}$.
By Lemma 13 in \cite{vaart2011information}, there exists a test $\phi$ satisfying
\begin{align*}
    \Ep_{n,f^*}[\phi_r] \leq 9 \mN(r/2, \mF_{NN,\eta}(S,B,L), \|\cdot\|_n) \exp(-r^2/8),
\end{align*}
and
\begin{align*}
    \sup_{f \in \mF_{NN,\eta}(S,B,L): \|f - f^*\|_n \geq r}\Ep_{n,f}[1-\phi_r] \leq \exp(-r^2/8),
\end{align*}
for any $r > 0$ and $j \in \N$.
By the entropy bound for $\mN(r, \mF_{NN,\eta}(S,B,L), \|\cdot\|_n) \leq \mN(r, \mF_{NN,\eta}(S,B,L), \|\cdot\|_{L^\infty})$, we have
\begin{align*}
    \Ep_{n,f^*}[\phi_{r}] \leq r^{-1}18 (L+1)N^2 \exp(-r^2/8 + S+1 ).
\end{align*}

\textbf{Step 2. Bound an error with fixed design.}

To evaluate contraction of the posterior distribution, we decompose the expected posterior distribution as
\begin{align*}
    &\Ep_{f^*} \left[ \Pi_f\left(f: \|f -f^*\|_n \geq 4 \epsilon r | \mD_n\right) \right] \\
    &\leq \Ep_{f^*} \left[\phi_r \right] + \Ep_{f^*} \left[ \mE_r^c \right]  + \Ep_{f^*} \left[  \Pi_f(f:\|f - f^*\|_n > 4 \epsilon r | \mD_n)(1-\phi_r) \mone_{\mE_r} \right] \\
    & =: A_n + B_n + C_n.
\end{align*}
Here, note that a support of $\Pi_f$ is included in $\mF_{NN,\eta}(S,B,L)$ due to the setting of $\Pi$.

About $A_n$, we use the bound about $\phi_r$ substitute $\sqrt{n}\epsilon r$ into $r$, then obtain
\begin{align*}
    A_n \leq 18 (\sqrt{n} \epsilon r)^{-1} (L+1)N^2 \exp(-n \epsilon^2 r^2/8 + S+1 ).
\end{align*}

About $B_n$, by using the result of $\mE_r$ as \eqref{ineq:event} and substitute $\sqrt{n}\epsilon r$ into $r$, then we have
\begin{align*}
    B_n \leq\exp(-n9 B_p^2 \epsilon_n^2 /8) + \exp(-nB_p^2 \epsilon_n^2 / B_f^2).
\end{align*}

About $C_n$, we decompose the term as
\begin{align*}
    C_n &= \Ep_{X} \left[ \Ep_{n,f^*} \left[  \frac{\int_{\mF_{NN,\eta}(S,B,L)}\mone_{\{\|f - f^*\|_n > 4 \epsilon r\}} p_{n,f}(Y_{1:n}) d\Pi_f(f)} {\int_{\mF_{NN,\eta}(S,B,L)}p_{n,f}(Y_{1:n})d \Pi_f(f)} (1-\phi_r) \mone_{\mE_r}\right] \right]\\
    &= \Ep_{X} \left[\Ep_{n,f^*} \left[  \frac{\int_{\mF}\mone_{\{\|f - f^*\|_n > 4 \epsilon r\}} \frac{p_{n,f}(Y_{1:n})}{p_{n,f^*}(Y_{1:n})}  d\Pi_f(f)} {\int_{\mF}\frac{p_{n,f}(Y_{1:n})}{p_{n,f^*}(Y_{1:n})}d \Pi_f(f)} (1-\phi_r) \mone_{\mE_r}\right] \right]\\
    &\leq  \Ep_{X} \Biggl[\Ep_{n,f^*} \Biggl[ \int_{f \in \mF_{NN,\eta}(S,B,L): \|f - f^*\|_n > \sqrt{2}\epsilon r} \frac{p_{n,f}(Y_{a:n})}{p_{n,f^*}(Y_{1:n})}d \Pi_f(f)\\
    &\quad \quad \quad \quad \quad \quad \times \exp(n\epsilon^2r^2)\Pi_f(f:\|f - \dot{f}\|_{L^\infty} < B_p \epsilon_n)^{-1} (1-\phi_r) \mone_{\mE_r}  \Biggr] \Biggr]\\
    &=  \Ep_{X} \Biggl[\Ep_{n,f^*} \Biggl[ \int_{f \in \mF_{NN,\eta}(S,B,L): \|f - f^*\|_n > \sqrt{2}\epsilon r} \frac{p_{n,f}(Y_{a:n})}{p_{n,f^*}(Y_{1:n})}d \Pi_f(f)\\
    & \quad \quad \quad \quad \quad \quad \times \exp(n\epsilon^2r^2 - \log \Pi_f(f:\|f - \dot{f}\|_{L^\infty} < B_p \epsilon_n)) (1-\phi_r) \mone_{\mE_r}  \Biggr] \Biggr]
\end{align*}
by the definition of $\mE_r$.
Here, we evaluate $- \log \Pi_f(f:\|f - \dot{f}\|_{L^\infty} < B_p \epsilon_n)$ as
\begin{align*}
    &- \log \Pi_f(f:\|f - \dot{f}\|_{L^\infty} < B_p \epsilon_n)  \leq - \log \Pi_\Theta(\Theta: \|\Theta - \dot{\Theta}\|_\infty < L_f B_p \epsilon_n) \leq S \log ((B_f L_f \epsilon_n)^{-1}),
\end{align*}
where $\dot{\Theta}$ is the parameter which constitute $\dot{f}$ and $L_f$ is a Lipschitz constant of $G_{\eta}[\cdot]$.
Thus, the bound for $C_n$ is rewritten as
\begin{align*}
    C_n&\leq \Ep_{X}\Biggl[ \int_{f \in \mF_{NN,\eta}(S,B,L): \|f - f^*\|_n > \sqrt{2}\epsilon r} \frac{p_{n,f}(Y_{1:n})}{p_{n,f^*}(Y_{1:n})}\Ep_{n,f} \left[  (1-\phi_r) \mone_{\mE_r}  \right]d \Pi_f(f)\\
    & \quad \quad \quad \quad \quad \quad \times \exp(n\epsilon^2r^2 + S \log ((B_f L_f \epsilon_n)^{-1}))\Biggr] \\
    & \leq \exp\left(n\epsilon^2 r^2 + S \log ((B_f L_f \epsilon_n)^{-1}) - \frac{r'^2}{8}\right),
\end{align*}
here, we introduce $r'$ is a $r$ for defining $\phi_r$ to identify $r$ for $\mE_r$.
Here, we substitute $r' = 4 \sqrt{n}\epsilon r$, then we have
\begin{align*}
    C_n & \leq \exp\left( S \log ((B_f L_f \epsilon_n)^{-1}) - 2 n\epsilon^2r^2\right)
\end{align*}


Combining the results about $A_n,B_n,C_n$ and $D_n$, we obtain
\begin{align*}
    &\Ep_{f^*}[\Pi_f(f : \|f - f^*\|_n \geq 4 \epsilon r | \mD_n)] \\
    &\leq  \exp(-n \epsilon^2 r^2/8 + S+1 + \log 18 (\sqrt{n} \epsilon r)^{-1} (L+1)N^2) \\
    & \quad  + \exp(-n9 B_p^2 \epsilon_n^2 /8) + \exp(-nB_p^2 \epsilon_n^2 / B_f^2) + \exp\left( S \log ((B_f L_f \epsilon_n)^{-1}) - 2 n\epsilon^2r^2\right)\\
    & \leq 2\exp\left(  - \max\{9B_p^2/8, B_p^2/B_f^2\}  n\epsilon_n^2 \right) \\
    & \quad + 2\exp\left( 2n \epsilon^2 r-2  + C''_S \max\{n^{-D / (2\beta + D)},n^{-2D-2 / (2\alpha + 2D-2)}\}) \log n+ 1 \right).
\end{align*}
by substituting the order or $S$ as \eqref{select:param} as $S = C'_S\max\{n^{-D / (2\beta + D)},n^{-2D-2 / (2\alpha + 2D-2)}\}) $ where $C'_S = C_S M(1 +  J(2^D + Q) $ and $C''_S$ is a constant as $C''_S =C'_S \log \max\{-D / (2\beta + D),-2D-2 / (2\alpha + 2D-2)\}) / (B_fL_f)$.
By substituting $r=1$ and
\begin{align*}
    \epsilon = \epsilon_n \log n = 2JM(2^D +Q - 1/2) \max\{n^{-\beta/ (2\beta + D)},n^{-\alpha / (\alpha + 2D-2)}\} \log n,
\end{align*}
then we obtain 
\begin{align*}
    \Ep_{f^*}\left[\Pi_f\left(f : \|f - f^*\|_n \geq C_\epsilon  \max\{n^{-\beta/ (2\beta + D)},n^{-\alpha / (\alpha + 2D-2)}\} \log n | \mD_n\right)\right] \to 0,
\end{align*}
as $n \to \infty$ with a constant $C_{\epsilon} > 0$

\subsection{The bound with a $L^2(P_X)$ norm}

We evaluate an expectation of the posterior distribution with respect to the $\|\cdot\|_{L^2(P_X)}$ norm.
The term is decomposed as
\begin{align*}
    &\Ep_{f^*} \left[ \Pi_f(f : \|f - f^*\|_{L^2(P_X)} >  r\epsilon | \mD_n ) \right] \\
    &\leq \Ep_{f^*} \left[\mone_{\mE_r^c}\right]  + \Ep_{f^*}\left[\mone_{\mE_r} \Pi_f(f: 2 \|f - f^*\|_n > r \epsilon  | \mD_n)\right] \\
    &\quad \quad + \Ep_{f^*}\left[\mone_{\mE_r} \Pi_f(f: 2 \|f - f^*\|_{L^2(P_X)} > r \epsilon > \|f - f^*\|_n  | \mD_n)\right] \\
    &=: I_n + II_n + III_n.
\end{align*}
for all $\epsilon > 0$ and $r > 0$.
Since we already bound $I_n$ and $II_n$ in step 2, we will bound $III_n$.

To bound the empirical norm, we provide the following lemma.
\begin{lemma}\label{lem:talagrand}
    Let a finite constant $B_f>0$ satisfy $B_f \geq \|\dot{f} - f^*\|_{L^\infty}$.
    Then, for any $r > 0$ and $f \in \mF_{NN,\eta}(S,B,L)$, we have
    \begin{align*}
        1-\exp(-nr^2 / B_f^2) \leq \Pr_X \left( \|f - f^*\|_n \leq \|f - \dot{f}\|_{L^\infty} + B_p \|\dot{f} - f^*\|_{L^2} + r \right).
    \end{align*}
\end{lemma}
\begin{proof}

We note that the finite $B_f$ exists.
We know that $\dot{f} \in \mF_{NN,\eta}(S,B,L)$ is bounded by Lemma \ref{lem:bounded}.
Also, $f^* \in \mF_{M,J,\alpha,\beta}$ is bounded since it is a finite sum of continuous functions with compact supports.

We evaluate $\|f - f^*\|_n $ as
\begin{align*}
    \|f - f^*\|_n \leq \|f - \dot{f}\|_n + \|\dot{f} - f^*\|_n \leq \|f - \dot{f}\|_{L^\infty} + \|\dot{f} - f^*\|_n.
\end{align*}
To bound the term $\|\dot{f} - f^*\|_n$, we apply the Hoeffding's inequality and obtain
\begin{align*}
    1-\exp(-2nr^2 / 2B_f^2) \leq \Pr_X \left( \|\dot{f} - f^*\|_n \leq  \|\dot{f} - f^*\|_{L^2(P_X)} + r \right).
\end{align*}
Using the inequality \eqref{ineq:bound_l2p}, we have
\begin{align*}
    \Pr_X \left( \|\dot{f} - f^*\|_n \leq  \|\dot{f} - f^*\|_{L^2(P_X)} + r \right) \leq \Pr_X \left( \|f - f^*\|_n \leq  B_p\|\dot{f} - f^*\|_{L^2} + r \right),
\end{align*}
then obtain the desired result.

\end{proof}

By Lemma \ref{lem:talagrand}, we know the bound
\begin{align*}
    1-\exp(-2nr'^2 / 2B_f^2) \leq \Pr_X \left( \|f - f^*\|_n \leq  \|f - f^*\|_{L^2(P_X)} + r' \right),
\end{align*}
for all $f$ such as $\|f\|_{L^\infty} \leq B$.
WE set $r' =  \|f - f^*\|_{L^2(P_X)} $, hence
\begin{align*}
    1-\exp\left(-\frac{n\|f - f^*\|_{L^2(P_X)}^2 }{ B_f^2}\right) \leq \Pr_X \left( \|f - f^*\|_n \leq  2\|f - f^*\|_{L^2(P_X)} \right).
\end{align*}
Using this result, we obtain
\begin{align*}
    III_n &\leq \Ep_{X} \left[ \Ep_{n,f^*} \left[  \int_{f \in \mF_{NN,\eta}(S,B,L) : \|f - f^*\|_{L^2(P_X)} > r \epsilon > 2 \|f - f^*\|_n} \frac{ p_{n,f}(Y_{1:n})}{p_{n,f^*}(Y_{1:n})} d\Pi_f(f) \mone_{\mE_r}\right] \right] \\
    & \quad \quad \quad \quad \times \exp\left(n \epsilon^2 r''^2 - \log \Pi_f(f : \|f - \dot{f}\|_{L^\infty} < B_p \epsilon_n)\right) \\
    & \leq \int_{f \in \mF_{NN,\eta}(S,B,L) : \|f - f^*\|_{L^2(P_X)} > r \epsilon} \Pr_X \left( \|f - f^*\|_{L^2(P_X)}  > 2 \|f - f^*\|_n\right) d \Pi_f(f)  \\
    &\quad \quad \quad \quad \times \exp\left(n \epsilon^2 r^2 +S \log ((B_f L_f \epsilon_n)^{-1})\right) \\
    & \leq \exp\left(n \epsilon^2 r''^2 + S \log ((B_f L_f \epsilon_n)^{-1}) - \frac{n r^2 \epsilon^2}{B_f^2} \right),
\end{align*}
where $r''$ is a parameter for defining $\mE_r$.
We substitute $r'' = r/\sqrt{2}B$, then we have
\begin{align*}
    III_n \leq  \exp\left(S \log ((B_f L_f \epsilon_n)^{-1}) - \frac{n r^2 \epsilon^2}{2B_f^2} \right)
\end{align*}
Following the same discussion in Section \ref{sec:bayes_fixed}, we combine the result and obtain
\begin{align*}
    &I_n + II_n + III_n \\
    &\leq 3\exp\left(  - \max\{9B_p^2/8, B_p^2/B_f^2\}  n\epsilon_n^2 \right) + \exp\left(S \log ((B_f L_f \epsilon_n)^{-1}) - n r^2 \epsilon^2/2B_f^2\right)\\
    & \quad + 3\exp\left( 2n \epsilon^2 r-2  + C''_S \max\{n^{-D / (2\beta + D)},n^{-2D-2 / (2\alpha + 2D-2)}\}) \log n+ 1 \right),
\end{align*}
and setting 
\begin{align*}
    \epsilon = \epsilon_n \log n = 2JM(2^D +Q - 1/2) \max\{n^{-\beta/ (2\beta + D)},n^{-\alpha / (\alpha + 2D-2)}\} \log n,
\end{align*}
yields the same results.

\qed

\section{Proof of Theorem \ref{thm:minimax}}

We discuss minimax optimality of the estimator.
We apply the techniques developed by \cite{yang1999information} and utilized by \cite{raskutti2012minimax}.

Let $\tilde{\mF}_{M,J,\alpha,\beta}(\delta) \subset  \mF_{M,J,\alpha,\beta}$ be a packing set of $\mF_{M,J,\alpha,\beta}$ with respect to $\|\cdot\|_{L^2}$, namely, each pair of elements $f,f' \in \tilde{\mF}_{M,J,\alpha,\beta}$ satisfies $\|f - f'\|_{L^2} \geq \delta$.
Following the discussion by \cite{yang1999information}, the minimax estimation error is lower bounded as
\begin{align*}
    &\min_{\bar{f}} \max_{f^* \in  \mF_{M,J,\alpha,\beta}} \Pr_{f^*} \left( \|\bar{f} - f^* \|_{L^2(P_X)} \geq \frac{\delta_n}{2} \right) \geq \min_{\bar{f}} \max_{f^* \in \tilde{\mF}_{M,J,\alpha,\beta}(\delta)} \Pr_{f^*} \left( \|\bar{f} - f^* \|_{L^2(P_X)} \geq  \frac{\delta_n}{2} \right).
\end{align*}
Let $\tilde{f}' := \argmin_{f' \in  \tilde{\mF}_{M,J,\alpha,\beta}(\delta)}\|\tilde{f} - \bar{f}\|$ be a projected estimator $\bar{f}$ onto  $\tilde{\mF}_{M,J,\alpha,\beta}(\delta)$.
Then, the value is lower bounded as
\begin{align*}
    &\min_{\bar{f}} \max_{f^* \in \tilde{\mF}_{M,J,\alpha,\beta}(\delta)} \Pr_{f^*} \left( \|\bar{f} - f^* \|_{L^2(P_X)} \geq  \frac{\delta_n}{2} \right) \\
    &\geq \min_{\bar{f}'} \max_{f \in \tilde{\mF}_{M,J,\alpha,\beta}(\delta) } \Pr_{f} (f \neq \bar{f}') \\
    & \geq \min_{\bar{f}'} \Pr_{\check{f} \sim U}(\bar{f}' \neq \check{f}),
\end{align*}
where $\check{f}$ is uniformly generated from $\tilde{\mF}_{M,J,\alpha,\beta}(\delta)$ and $Pr_U$ denotes a probability with respect to the uniform distribution.

We apply the Fano's inequality (summarized as Theorem 2.10.1 in \cite{cover2012elements}), we obtain
\begin{align*}
    \Pr_{\check{f} \sim U}(\bar{f}' \neq \check{f}) \geq 1 - \frac{I( F_U; D_n) + \log 2}{\log |\tilde{\mF}_{M,J,\alpha,\beta}(\delta)|},
\end{align*}
where $I( F_U; Y_{1:n})$ is a mutual information between a uniform random variable $F_U$ on $ \tilde{\mF}_{M,J,\alpha,\beta}(\delta)$ and $Y_{1:n}$.
The mutual information is evaluated as
\begin{align*}
    & I( F_U; Y_{1:n})\\
    &= \frac{1}{|\tilde{\mF}_{M,J,\alpha,\beta}(\delta')|} \sum_{f \in \tilde{\mF}_{M,J,\alpha,\beta}(\delta')} \int \log \left( \frac{p_{n,f}(Y_{1:n})}{E_{F_U}[p_{n,F_U}(Y_{1:n})]} \right) dP_{n,f}(Y_{1:n})\\
    &\leq \max_{f \in \tilde{\mF}_{M,J,\alpha,\beta}(\delta')} \int \log \left( \frac{p_{n,f}(Y_{1:n})}{E_{F_U}[p_{n,F_U}(Y_{1:n})]} \right) dP_{n,f}(Y_{1:n}) \\
    & \leq \max_{f \in \tilde{\mF}_{M,J,\alpha,\beta}(\delta')} \max_{f' \in \tilde{\mF}_{M,J,\alpha,\beta}(\delta')}\int \log \left( \frac{p_{n,f}(Y_{1:n})}{|\tilde{\mF}_{M,J,\alpha,\beta}(\delta')|^{-1}p_{n,f'}(Y_{1:n})} \right) dP_{n,f}(Y_{1:n}) \\
    & =  \max_{f,f' \in \tilde{\mF}_{M,J,\alpha,\beta}(\delta')} \log |\tilde{\mF}_{M,J,\alpha,\beta}(\delta')| +  \int \log \left( \frac{p_{n,f}(Y_{1:n})}{p_{n,f'}(Y_{1:n})} \right) dP_{n,f}(Y_{1:n}).
\end{align*}
Here, we know that 
\begin{align*}
     \log |\tilde{\mF}_{M,J,\alpha,\beta}(\delta')| \leq \log \mN(\delta'/2, \mF_{M,J,\alpha,\beta},\|\cdot\|_{L^2}),
\end{align*}
and 
\begin{align*}
     \int \log \left( \frac{p_{n,f}(Y_{1:n})}{p_{n,f'}(Y_{1:n})} \right) dP_{n,f}(Y_{1:n}) \leq \frac{n}{2}\Ep_X\left[ \|f - f'\|_n^2 \right] \leq \frac{n}{8}\delta'^2,
\end{align*}
since $f,f' \in \tilde{\mF}_{M,J,\alpha,\beta}(\delta')$.


We will provide a bound for $\log \mN(\delta'/2, \mF_{M,J,\alpha,\beta},\|\cdot\|_{L^2})$.
Since $\mF_{M,J,\alpha,\beta}$ is a sum of $M$ functions in $\mF_{1,J,\alpha,\beta}$, we have
\begin{align*}
    \log \mN(\delta, \mF_{M,J,\alpha,\beta},\|\cdot\|_{L^2}) \leq M\log \mN(\delta', \mF_{1,J,\alpha,\beta},\|\cdot\|_{L^2}).
\end{align*}
To bound $\log \mN(\delta', \mF_{1,J,\alpha,\beta},\|\cdot\|_{L^2})$, we define $\mI_{\alpha,J} := \{\mone_{R}:I^D \to \{0,1\}| R \in \mR_{\alpha,J}\}$.
We know that $\mF_{1,J,\alpha,\beta} = H^{\beta}(I^D) \otimes \mI_{\alpha,J}$, hence we obtain
\begin{align*}
    &\log \mN(\delta', \mF_{1,J,\alpha,\beta},\|\cdot\|_{L^2}) \leq \log \mN(\delta',H^{\beta}(I^D),\|\cdot\|_{L^2})  + \log \mN(\delta', \mI_{\alpha,J},\|\cdot\|_{L^2}) .
\end{align*}
By the entropy bound for smooth functions (e.g. Theorem 2.7.1 in \cite{van1996weak}), we use the bound
\begin{align*}
    \log \mN(\delta',H^{\beta}(I^D),\|\cdot\|_{L^2}) \leq C_H \delta'^{-D / \beta},
\end{align*}
with a constant $C_H > 0$.
Furthermore, about the covering number of $\mI_{\alpha, J}$, we use the relation
\begin{align*}
    \|\mone_R - \mone_{R'}\|_{L^2}^2 &= \int (\mone_R(x) - \mone_{R'}(x))^2 dx = \int (\mone_R(x) - \mone_{R'}(x)) dx \\
    &=\int_{\textbf{x} \in  I^D} \mone_{R}(\textbf{x})(1- \mone_{R'}(\textbf{x}))  d\textbf{x} =: d_1(R,R'),
\end{align*}
where $R,R' \in \mR_{\alpha,J}$ and $d_1$ is a difference distance with a Lebesgue measure for sets by \cite{dudley1974metric}.
By Theorem 3.1 in \cite{dudley1974metric}, we have
\begin{align*}
    \log \mN(\delta', \mR_{\alpha,J}, d_1) \leq C_{\lambda} \delta'^{-(D-1)/\alpha},
\end{align*}
with a constant $C_\lambda > 0$.
Then, we bound the entropy of $\mI_{\alpha, J}$ as
\begin{align}
    \log \mN(\delta', \mI_{\alpha,J}, \|\cdot\|_{L^2}) = \log \mN(\delta'^2, \mR_{\alpha,J}, d_1). \label{eq:ent_ind_set}
\end{align}
To bound the term, we provide the following Lemma.
\begin{lemma} \label{lem:ent_set}
	We obtain
    \begin{align*}
    	\log \mN(\delta, \mR_{\alpha,J},d_1) \leq J  \mN(\delta/J, \mR_{\alpha,1},d_1).
    \end{align*}
\end{lemma}
\begin{proof}
Fix $\delta > 0$ arbitrary.
Let $\tilde{\mR} \subset \mR_{\alpha,1}$ be a centers of the $\delta$-covering balls, and $|\tilde{\mR}| = \bar{R}$.
Also define that $\tilde{\mR}^J := \{\cap_{j \in [J]}R_j \mid R_j \in \tilde{\mR}\}$.
Obviously, we have $\tilde{\mR}^J \subset \mR_{\alpha,J}$ and $|\tilde{\mR}^J | = \bar{R}^J$.

Consider $R \in \mR_{\alpha,J}$.
By its definition, there exist $\tilde{R}_1,...,\tilde{R}_J \in \mR_{\alpha,1}$ and satisfy $R = \cap_{j \in [J]} \tilde{R}_j$.
Since $\tilde{\mR}$ is a set of centers of the covering balls, there exist $\dot{R}_1,...,\dot{R}_J \in \tilde{\mR}$ and $d_1(\dot{R}_j,\tilde{R}_j) \leq \delta$ holds.

Here, we define $\dot{R} \in \tilde{\mR}^J $ as $\dot{R} = \cap_{j \in [J]} \dot{R}_j$.
Now, we have 
\begin{align*}
	d_1(\dot{R},R) \leq \sum_{j \in [J]} d_1(\dot{R}_j,\tilde{R}_j) \leq J \delta.
\end{align*}
Hence, for arbitrary $R \in \mR_{\alpha,J}$, there exists $\dot{R}$ in $\tilde{\mR}^J $ and their distance is bounded by $J\delta$.
Now, we can say that $\tilde{\mR}^J$ is a set of centers for covering balls for $\mR_{\alpha,J}$ with radius $J\delta$.
Since $|\tilde{\mR}^J | = \bar{R}^J$, the statement holds.
\end{proof}

Applying Lemma \ref{lem:ent_set}, we obtain
\begin{align*}
   \log \mN(\delta'^2, \mR_{\alpha,J}, d_1) \leq  C_{\lambda} J^{(\alpha + D-1)/\alpha}\delta'^{-2(D-1)/\alpha}.
\end{align*}
Substituting the results yields
\begin{align*}
    \log \mN(\delta'/2, \mF_{M,J,\alpha,\beta},\|\cdot\|_{L^2}) \leq MC_H\delta'^{-D / \beta} + MC_{\lambda} \delta'^{-2(D-1)/\alpha} .
\end{align*}


We provide a lower bound for $\log |\tilde{\mF}_{M,J,\alpha,\beta}(\delta)|$.
Let $D(\delta,\mF_{M,J,\alpha,\beta}, \|\cdot\|_{L^2} )$ be a notation for a packing number $|\tilde{\mF}_{M,J,\alpha,\beta}(\delta)|$. 
Now, we have
\begin{align*}
	&\log D(\delta,\mF_{M,J,\alpha,\beta}, \|\cdot\|_{L^2} ) \geq \log D(\delta,\mF_{1,J,\alpha,\beta}, \|\cdot\|_{L^2} )\\
    &\quad \quad \geq \max\{ \log D(\delta,H^\beta(I^D), \|\cdot\|_{L^2} ), \log D(\delta,\mI_{\alpha,J}, \|\cdot\|_{L^2} ) \}.
\end{align*}
Similar to \eqref{eq:ent_ind_set}, 
\begin{align*}
    & \log D(\delta,\mI_{\alpha,J}, \|\cdot\|_{L^2} ) = \log D(\delta^2, \mR_{\alpha,J}, d_1) \geq \log D(\delta^2, \mR_{\alpha,1}, d_1).    
\end{align*}

About $\log D(\delta,H^\beta(I^D), \|\cdot\|_{L^2} )$, we apply Lemma 3.5 in \cite{dudley1974metric} then 
\begin{align*}
	\log D(\delta,H^\beta(I^D), \|\cdot\|_{L^2} ) \geq \log \mN(\delta,H^\beta(I^D), \|\cdot\|_{L^2} ) \geq c_{lh} \delta^{- D/\beta},
\end{align*}
with some constant $c_{lh} >0$.
About $\log D(\delta^2, \mR_{\alpha,1}, d_1)$, since the definition of $\mR_{\alpha,1}$ follows the boundary fragmented class by restricting sets as a image of smooth embeddings, we apply Theorem 3.1 in \cite{dudley1974metric} and obtain
\begin{align*}
	\log D(\delta^2, \mR_{\alpha,1}, d_1)\geq \log \mN(\delta^2, \mR_{\alpha,1}, d_1) \geq c_{lr} \delta^{- 2(D-1)/\alpha},
\end{align*}
with some constant $c_{lr} >0$.

Then, we provide a lower bound of $\Pr_{\check{f} \sim U}(\bar{f}' \neq \check{f})$ as
\begin{align*}
    \Pr_{\check{f} \sim U}(\bar{f}' \neq \check{f}) \geq 1- \frac{C_HM \delta'^{-D/\beta} + C_\lambda M \delta'^{-2(D-1)/\alpha}+\frac{n}{2}\delta^2 + \log 2}{\max\{ c_{lh} \delta^{-D/\beta}, c_{lr} \delta^{-2(D-1)/\alpha}\}}.
\end{align*}
By selecting $\delta $ and $\delta'$ as having an order $ \max\{n^{-2\beta/(2\beta + D)}, n^{-\alpha / (\alpha + 2D - 2)}\}$ and satisfying  
\begin{align*}
	1- \frac{C_HM \delta'^{-D/\beta} + C_\lambda M \delta'^{-2(D-1)/\alpha}+\frac{n}{2}\delta^2 + \log 2}{\max\{ c_{lh} \delta^{-D/\beta}, c_{lr} \delta^{-2(D-1)/\alpha}\}} \geq \frac{1}{2}.
\end{align*}
Then, we finally obtain the statement of Theorem \ref{thm:minimax}.

\section{Specific Examples of Other Inefficient Methods}

Orthogonal series methods estimate functions using an orthonormal basis.
It is one of the most fundamental methods for nonparametric regression (For an introduction, see Section 1.7 in \cite{tsybakov2003introduction}).
Let $\phi_j(x)$ for $j \in \N$ be an orthonormal basis function in $L^2(P_X)$.
An estimator for $f^*$ by the orthogonal series method is defined as
\begin{align*}
    \hat{f}^S(x) := \sum_{j \in [J]} \hat{\gamma}_j \phi_j(x),
\end{align*}
where $J \in \N$ is a hyper-parameter and $\hat{\gamma}_j$ is a coefficient calculated as $\hat{\gamma}_j := \frac{1}{n} \sum_{i \in [n]} Y_i \phi_j(X_i)$.
When the true function is smooth, i.e. $f^* \in H^\beta$, $\hat{f}^S$ is known to be optimal in the minimax sense \cite{tsybakov2003introduction}.
About estimation for $f^* \in \mF_{M,J,\alpha,\beta}$, we can obtain the following proposition.
\begin{proposition} \label{prop:fourier}
    Fix $D \in \N \backslash \{1\} , M, J \in \N, \alpha > 2$ and $\beta > 1$ arbitrary.
    Let $\hat{f}^S$ be the estimator by the orthogonal series method.
    Suppose $\phi_j, j \in \N$ are the trigonometric basis or the Fourier basis. 
    Then, with sufficient large $n$, there exist $f^* \in \mF_{M,J,\alpha,\beta}$, $P_X$, a constant $C_F > 0$, and a parameter
    \begin{align*}
         - \kappa > \max\{-2\beta/(2\beta + D), -\alpha/(\alpha + D - 1)\},
\end{align*}        
    such that
    \begin{align*}
        \Ep_{f^*}\left[ \|\hat{f}^F - f^*\|_{L^2(P_X)}^2 \right] > C_F n^{-\kappa}.
    \end{align*}
\end{proposition}
\begin{proof}

We will specify $f^* \in \mF_{M, J,\alpha,\beta}$ and distribution of $X$, and derive an rate of convergence by the estimator by the Fourier method.

For preparation, we consider $D=1$ case.
Let $X$ be generated by a distribution which realize a specific case $X_i = i/n$.
Also, we specify $f^* \in \mF_{M,J,\alpha,\beta}$ as
\begin{align*}
    f^*(x) = \mone_{\{x_1 \geq 0.5\}},
\end{align*}
with $x = (x_1,x_2) \in I^2$.
We consider a decomposition of $f^*$ by the trigonometric basis such as
\begin{align*}
    \phi_j(x) = 
    \begin{cases}
        1 &\mbox{~if~}j=0,\\
         \sqrt{2} \cos (2\pi k x) &\mbox{~if~}j = 2k,\\
         \sqrt{2} \sin (2\pi k x) &\mbox{~if~}j = 2k + 1,
    \end{cases}
\end{align*}
for $k \in \N$.
Then, we obtain
\begin{align*}
    f^* = \sum_{j \in \N \cup \{0\} } \theta_{j}^*\phi_{j}.
\end{align*}
Here, $\theta_{j}^*$ is a true coefficient.

For the estimator, we review its definition as follows.
The estimator is written as
\begin{align*}
    \hat{f}^F = \sum_{j\in [J]\cup \{0\}}\hat{\theta}_{j} \phi_{j},
\end{align*}
where $\hat{\theta}_{j_1,j_2}$ is a coefficient which is defined as
\begin{align*}
    \hat{\theta}_{j} = \frac{1}{n} \sum_{i \in [n]} Y_i \phi_{j}(X_i).
\end{align*}
Also, $J \in \N$ are hyper-parameters.
Since $\phi_{j}$ is an orthogonal basis in $L^2$ and the Parseval's identity, an expected loss by the estimator is decomposed as
\begin{align*}
    \Ep_{f^*} \left[\|\hat{f}^F - f^* \|_{L^2(P_X)}^2 \right] &= \Ep_{f^*} \left[\sum_{j \in \N \cup \{0\}} ( \hat{\theta}_{j} - \theta_{j}^*)^2 \right] \\
    &=\Ep_{f^*} \left[\sum_{j \in [J]\cup \{0\}} ( \hat{\theta}_{j} - \theta_{j}^*)^2 + \sum_{j > J}(\theta_{j}^*)^2 \right] \\
    &=\sum_{j \in [J]\cup \{0\}}\Ep_{f^*} \left[( \hat{\theta}_{j}- \theta_{j}^*)^2   \right]+ \sum_{j>J}(\theta_{j}^*)^2.
\end{align*}
Here, we apply Proposition 1.16 in \cite{tsybakov2003introduction} and obtain
\begin{align*}
    \Ep_{f^*} \left[\|\hat{f}^F - f^* \|_{L^2(P_X)}^2 \right] &=\sum_{j \in [J]\cup \{0\}}\left( \frac{\sigma^2}{n} + \rho_{j}^2 \right) + \sum_{j > J}(\theta_{j}^*)^2\\
    &\geq \sum_{j \in [J]\cup \{0\}}\frac{\sigma^2}{n} + \sum_{j > J}(\theta_{j}^*)^2 \\
    &= \frac{\sigma^2 (J+1)}{n} + \sum_{j>J}(\theta_{j}^*)^2,
\end{align*}
where $\rho_{j} := n^{-1}\sum_{i \in [n]}f(X_i) \phi_{j}(X_i) - \langle f, \phi_{j} \rangle$ is a residual.

Considering the Fourier transform of step functions, we obtain $\theta_{j}^* = \frac{1-(-1)^j}{2\pi j}$, hence
\begin{align*}
    \sum_{j> J}(\theta_{j}^*)^2 = \frac{1}{4 \pi^2} \Psi(J+1) = \frac{1}{4 \pi^2} \sum_{k \in \N \cup \{0\}} \frac{1}{(J +1+ k)^2} \geq \frac{1}{4 \pi^2 (J+1)^2},
\end{align*}
where $\Psi$ is the digamma function.

Combining the results, we obtain
\begin{align*}
    \Ep_{f^*} \left[\|\hat{f}^F - f^* \|_{L^2(P_X)}^2 \right] \geq \frac{\sigma^2 J+1 }{n} + \frac{1}{4 \pi^2 (J+1)^2}.
\end{align*}
We set $J = \lfloor c_J n^{1/3} - 1 \rfloor $ with a constant $c_J > 0$, then we finally obtain
\begin{align*}
    \Ep_{f^*} \left[\|\hat{f}^F - f^* \|_{L^2(P_X)}^2 \right] \geq n^{-2/3} \left(\sigma^2 + \frac{1}{4 \pi^2} \right).
\end{align*}
Then, we obtain the lower bound for the $D=1$ case.

For general $D \in \N$, we set a true function as
\begin{align*}
    f^* = \bigotimes_{d \in [D]} \mone_{\{\cdot \geq 0.5\}}.
\end{align*}
Due to the tensor structure, we obtain the decomposed form
\begin{align*}
    f^* = \sum_{j_1 \in \N \cup \{0\}} \cdots \sum_{j_D \in \N \cup \{0\}} \gamma_{j_1 ,...,j_D} \bigotimes_{d \in [D]} \phi_{j_d},
\end{align*}
where $\gamma_{j_1 ,...,j_D}$ is a coefficient such as 
\begin{align*}
    \gamma_{j_1 ,...,j_D} = \prod_{d \in [D]} \theta_{j_d},
\end{align*}
using $\theta_{j_d}$ in the preceding part.
Following the same discussion, we obtain the following lower bound as
\begin{align*}
    \Ep_{f^*} \left[\|\hat{f}^F - f^* \|_{L^2(P_X)}^2 \right] \geq \frac{\sigma^2 (J+1)^D}{n} + D \sum_{j > J} (\theta_j^*)^2.
\end{align*}
Then, we set $J-1 = \lfloor n^{1/(2 + D)} \rfloor$, we obtain that the bound is written as
\begin{align*}
    \Ep_{f^*} \left[\|\hat{f}^F - f^* \|_{L^2(P_X)}^2 \right] \geq n^{-2/(2 + D)} \left(\sigma^2 + \frac{D}{2 \pi^2} \right).
\end{align*}
Then, we obtain the claim of the proposition for any $D \in N_{\geq 2}$.

\end{proof}

Proposition \ref{prop:fourier} shows that $\hat{f}^S$ can estimate $f^* \in \mF_{M,J,\alpha,\beta}$ consistently since the orthogonal basis in $L^2(P_X)$ can reveal all square integrable functions.  Its order is, however, strictly worse than the optimal order.
Intuitively, the method requires many basis functions to express the non-smooth structure of $f^* \in \mF_{M, J,\alpha,\beta}$, and a large number of bases increases variance of the estimator, hence they lose efficiency.

\end{document}